\documentclass[11pt,letterpaper]{article}
\usepackage[margin=1in]{geometry}
\usepackage[T1]{fontenc}
\usepackage[utf8]{inputenc}
\usepackage{mathptmx}
\usepackage{courier}
\usepackage{amsmath,amssymb,amsthm,bm,mathtools}
\usepackage{booktabs,tabularx,array}
\usepackage{graphicx}
\usepackage{microtype}
\usepackage[round]{natbib}

\usepackage{fancyhdr}
\usepackage{caption}
\usepackage{xcolor}
\usepackage[colorlinks=true,citecolor=blue!40!black,linkcolor=blue!40!black,urlcolor=blue!40!black]{hyperref}

\usepackage{float}
\IfFileExists{placeins.sty}{\usepackage[section]{placeins}}{\providecommand{\FloatBarrier}{}}
\newcommand{\E}{\mathbb{E}}
\newcommand{\Pp}{\mathbb{P}}
\newcommand{\CC}{\mathbb{C}}
\newcommand{\RR}{\mathbb{R}}

\newcommand{\HCN}{\mathrm{HCN}}
\newcommand{\NN}{\mathrm{NN4MGF}}
\newcommand{\MRE}{\mathrm{MRE}}

\newtheorem{theorem}{Theorem}
\newtheorem{proposition}{Proposition}
\newtheorem{lemma}{Lemma}[section]

\hypersetup{pdfauthor={Kehan Wang},pdftitle={Holomorphic Cauchy Networks for Efficient Stationary Tail Computation in High-Dimensional Reflected Brownian Motions}}

\begin{document}
\thispagestyle{plain}
\begin{center}
 {\large\bfseries Holomorphic Cauchy Networks for Efficient\\[3pt] Stationary Tail Computation in High-Dimensional\\[3pt] Reflected Brownian Motions\par}
 \vspace{10pt}
 {Kehan Wang\par}
 {\small School of Data Science, The Chinese University of Hong Kong, Shenzhen, China\par}
\end{center}
\vspace{7pt}
\begin{abstract}
Stationary tail probabilities quantify severe congestion and workload risks in stochastic networks, yet their computation for high-dimensional reflected Brownian motions (RBMs) remains challenging. Building on the shared-coordinate, basic-adjoint-relationship (BAR) transform-learning framework of NN4MGF, we develop Holomorphic Cauchy Networks (HCNs) to move holomorphicity from the loss function into the network architecture. Rather than coupling real and imaginary output heads through a sampled Cauchy--Riemann (CR) penalty, HCN embeds rational Cauchy activations into native complex-valued subnetworks that are holomorphic by construction on pole-free domains, avoiding the input-derivative graphs required by the CR loss. Exact origin anchors and boundary-coordinate exclusion enforce normalization and boundary-independence identities algebraically, anchored interaction subnetworks extend the coordinate-wise representation to correlated systems, and closed-form real-axis derivatives with sequential branch backpropagation control training memory. Under holomorphic-log assumptions, we establish simultaneous value and first-derivative approximation for separable targets with the shared-coordinate representation and for general constrained holomorphic targets with the interaction extension, while deriving a mixed-derivative lower bound showing why coordinate-only representations incur an irreducible error under nonzero cross-covariance. In a 30-dimensional profile on the same NVIDIA H100 NVL, the tested HCN configuration achieves a 9.7-fold reduction in mean update time and approximately tenfold lower peak allocated memory than NN4MGF, corresponding to a 58.2-fold projected full-schedule update-time ratio across their respective training schedules. Using 64 million sampled base inputs per high-dimensional run, all HCN experiments run on a single GPU and scale to 50 dimensions on a laptop GPU with 6.367\% mean relative error over thresholds 1 through 8. On the two-dimensional Harrison benchmark, HCN reduces mean relative error from 6.684\% for NN4MGF to 1.422\%. On the exploratory ten-dimensional correlated RBM benchmark, HCN achieves a mean relative error of 3.829\%, compared with 19.386\% for the independent-marginal baseline. Code is available at \url{https://github.com/KehanWang126/holomorphic-cauchy-networks-rbm}.
\end{abstract}
\vspace{4pt}
\vspace{7pt}
\section{Introduction}\label{sec:intro}
Reflected Brownian motions (RBMs) serve as diffusion models of congestion and workload dynamics in stochastic networks. Their stationary tail probabilities quantify the risk of large workloads and therefore add important risk information beyond mean-based performance measures. Explicit stationary laws are available only for special model classes, including product-form systems characterized by structural skew-symmetry conditions on the model parameters \citep{HarrisonWilliams1987,DaiMiyazawaWu2014}. For more general models, numerical approaches include methods based on the basic adjoint relationship (BAR) \citep{DaiHarrison1992} and high-dimensional steady-state simulation \citep{BlanchetEtAl2021}. Neural computation has also proved useful for high-dimensional RBM drift control \citep{AtaHarrisonSi2025}, for learning convergence rates and stationary distributions of Markov chains \citep{QuBlanchetGlynn2024,QuBlanchetGlynn2026}, and more broadly in equation-residual scientific machine learning \citep{RaissiEtAl2019}. As dimension increases, obtaining accurate stationary tail estimates at a manageable computational cost remains a central challenge.

A recent approach by \citet{DaiZhang2026}, NN4MGF, learns stationary and boundary Laplace transforms from the BAR and recovers tail probabilities through numerical Laplace inversion. This formulation allows neural networks to be trained directly from the RBM parameters $(\mu,\Sigma,R)$ without requiring samples from the stationary distribution. As noted by \citet{DaiZhang2026}, however, GPU memory demand and training time remain key obstacles to further scaling: processing large training batches in a single update requires substantial GPU memory, whereas splitting them across smaller batches increases runtime. Our aim is to address these structural and computational bottlenecks---substantially reducing GPU memory demand, update time, and sampling requirements while maintaining or improving tail accuracy---so that BAR-based transform learning scales to higher dimensions on modest hardware and extends to correlated non-product systems.

The complex-analytic structure of Laplace transforms offers a natural way to relieve this computational bottleneck. Numerical Laplace inversion evaluates transforms at complex arguments, where the true stationary and boundary transforms are holomorphic throughout the interior of their convergence domains. Within standard real-valued deep learning frameworks, \citet{DaiZhang2026} map the real and imaginary parts of each coordinate through Fourier features into two-headed real-valued networks $(g^{\rm Re}, g^{\rm Im})$ to approximate the real and imaginary parts of the log transform. Because two real-valued outputs do not automatically form a complex-differentiable pair, NN4MGF couples them during training via a sampled Cauchy--Riemann (CR) penalty. Treating holomorphicity as a training-time differential constraint is effective, but it requires computing first-order input derivatives and retaining their computational graphs in GPU memory during parameter backpropagation, adding memory overhead as dimension increases.

A natural alternative is to move holomorphicity from the loss function into the network architecture itself. Recently, \citet{LiXiaZhang2025} introduced a rational Cauchy activation function motivated by Cauchy's integral formula for real-function approximation, \citet{SiEtAl2026} demonstrated its effectiveness in physics-informed neural networks for real-valued PDEs, and \citet{ZhangLiYangXia2026} studied holomorphic reciprocal Cauchy-type activations (CauchyNet) for compact, data-efficient supervised learning and forecasting. Whereas those works address general function approximation, PDE residuals, or supervised regression, learning RBM Laplace transforms from the BAR functional equation requires coupling $d+1$ interior and boundary transforms on $\CC^d$ under exact normalization and face-independence constraints. Evaluated directly on complex affine transformations in $\CC$, the rational Cauchy activation $\sigma_r(z) = (a_r z + b_r)/(z^2 + s_r^2)$ is complex-differentiable wherever its denominator is nonzero. Building on this property, we construct native complex-valued Cauchy subnetworks and embed them into NN4MGF's shared-coordinate BAR framework to form Holomorphic Cauchy Networks (HCNs). Because HCN satisfies the Cauchy--Riemann equations identically by construction on pole-free domains, this design replaces separate real and imaginary output heads and Fourier feature encodings while removing the sampled CR penalty, thereby avoiding the input-derivative graphs required by the CR loss and preserving the strengths of the BAR-learning formulation.

Building on NN4MGF's parameter-sharing mechanism across coordinates, HCN incorporates several structural and computational refinements for the interior and boundary transforms. Additive coordinate blocks provide an economical representation of separable log transforms, while anchored interaction subnetworks extend the architecture to dependent systems. Exact origin anchoring enforces the known interior normalization and boundary masses, and boundary-coordinate exclusion ensures that each boundary representation is independent of the coordinate normal to its supporting face. These identities hold algebraically by construction rather than through additional soft penalty terms, and evaluating isolated boundary residuals directly in transform space avoids taking complex logarithms of BAR coefficients. To further control memory usage during training, we employ closed-form input derivatives for the real-axis constraints, sequential backpropagation across the BAR, boundary, and real-axis loss branches, and dimension-dependent batch schedules.

We establish approximation properties for the resulting constrained architecture (Theorem~\ref{thm:approx}). Under holomorphic-log assumptions, the shared-coordinate representation can simultaneously approximate function values and first derivatives of separable targets on compact polydiscs, while the anchored interaction extension approximates general holomorphic targets satisfying the prescribed origin and boundary-coordinate constraints. The construction preserves parameter sharing and the exact structural identities with a uniform denominator margin. Conversely, we derive a mixed-derivative lower bound showing that whenever stationary coordinates have nonzero cross-covariance, coordinate-only representations incur an irreducible uniform approximation error on the joint log transform that increasing width alone cannot remove.

Our numerical experiments evaluate the proposed framework across computational efficiency, hardware accessibility, and tail accuracy. In a 30-dimensional profile on the same NVIDIA H100 NVL, the tested HCN configuration achieves a 9.7-fold reduction in mean update time and approximately tenfold lower peak allocated memory compared with NN4MGF; extrapolating the measured update times to the respective scheduled update counts gives a projected full-schedule update-time ratio of 58.2. These ratios compare the complete computational configurations---including their architectures, batch sizes, derivative calculations, and graph schedules---rather than isolating the effect of CR penalty removal. Using 64 million sampled base inputs in each high-dimensional run, all HCN experiments can be completed on a single GPU and scale to 50 dimensions on a single laptop GPU, achieving 6.367\% mean relative error over thresholds 1 through 8 at 50 dimensions with comparable accuracy across final-stage checkpoints. Across the degree-matched comparisons, HCN attains lower reported mean relative errors than NN4MGF over the primary evaluation ranges, reducing error from 6.684\% to 1.422\% on the two-dimensional Harrison non-product benchmark, while an exploratory ten-dimensional correlated example illustrates the interaction extension in higher dimensions. Exact-transform inversion diagnostics and a complex-tanh control further motivate Degree~8 Fixed-Talbot inversion and the rational Cauchy activation.

Section~\ref{sec:rbm} introduces the RBM transform problem, Section~\ref{sec:method} develops HCN and its approximation theorem, and Section~\ref{sec:training} details the memory-controlled training computations. Section~\ref{sec:experiments} reports the numerical experiments, and Sections~\ref{sec:discussion}--\ref{sec:conclusion} discuss scope and conclusions. Proofs, implementation details, and additional diagnostics are provided in the appendices, and code is available at \url{https://github.com/KehanWang126/holomorphic-cauchy-networks-rbm}.

\section{RBM Stationary Tails and Transform Learning}\label{sec:rbm}
\subsection{The stationary tail problem}
A $d$-dimensional RBM in $\mathbb{R}_+^d$ satisfies
\begin{equation}\label{eq:rbm}
Z(t)=Z(0)+\mu t+B(t)+RY(t),\qquad t\ge0,
\end{equation}
where $\mu\in\RR^d$ is the drift vector, $B$ is a $d$-dimensional zero-mean Brownian motion with covariance matrix $t\Sigma$, and the columns $R^{(k)}$ of the $d\times d$ matrix $R$ specify the reflection directions on the boundary faces of the orthant. Each continuous regulator process $Y_k$ starts at zero, is nondecreasing, and increases only when the state hits the $k$th boundary face $Z_k=0$, satisfying $\int_0^T Z_k(u)\,dY_k(u)=0$ for all $T\ge0$.

We assume that $\Sigma$ is symmetric positive definite, $R$ is nonsingular with componentwise positive vector $-R^{-1}\mu$, and the RBM is well defined with a unique stationary distribution $\pi$ \citep{HarrisonWilliams1987}. Under stationary initialization $Z(0)\sim\pi$, the state vector has finite first moments and the regulator process has finite expected increments over unit time intervals. Our primary performance metric is the stationary tail probability of the total system workload:
\begin{equation}\label{eq:tail}
S=\bm{1}^{\mathsf T} Z(0),\qquad p(t)=\Pp_{\pi}(S>t),\qquad t\ge0.
\end{equation}
Our objective is to compute $p(t)$ directly from the system primitives $(\mu,\Sigma,R)$ without generating sample paths from $\pi$.

\subsection{Transforms and the basic adjoint relationship}
For a complex vector $\theta\in\CC^d$, define the interior and boundary Laplace transforms as
\begin{align}
\phi_0(\theta)&=\E_\pi\bigl[e^{-\theta^{\mathsf T}Z(0)}\bigr],\label{eq:phi0}\\
\phi_k(\theta)&=\E_\pi\biggl[\int_0^1 e^{-\theta^{\mathsf T}Z(u)}\,dY_k(u)\biggr],\qquad k=1,\ldots,d,\label{eq:phik}
\end{align}
where complex inner products use ordinary transposes without conjugation. These transforms are finite whenever $\Re\theta$ is componentwise nonnegative and are holomorphic throughout the interior of their common domain of absolute convergence. Each $\phi_k$ represents the Laplace transform of a finite boundary measure supported on the face $\{z\in\RR_+^d:z_k=0\}$, which renders $\phi_k(\theta)$ structurally independent of the $k$th coordinate $\theta_k$.

Stationarity of \eqref{eq:rbm} directly determines the transform values at the origin:
\begin{equation}\label{eq:mass}
m=\E_\pi[Y(1)]=-R^{-1}\mu,\qquad\phi_0(0)=1,\qquad\phi_k(0)=m_k,\qquad k=1,\ldots,d.
\end{equation}
Applying It\^o's formula to the test function $e^{-\theta^{\mathsf T}z}$ and taking expectations under $\pi$ yields the Laplace transform version of the basic adjoint relationship (BAR):
\begin{equation}\label{eq:bar}
\gamma_0(\theta)\phi_0(\theta)=\sum_{k=1}^d\gamma_k(\theta)\phi_k(\theta),
\end{equation}
with known quadratic and linear polynomial coefficients
\begin{equation}\label{eq:gammas}
\gamma_0(\theta)=-\tfrac12\theta^{\mathsf T}\Sigma\theta+\mu^{\mathsf T}\theta,\qquad
\gamma_k(\theta)=-\theta^{\mathsf T}R^{(k)},\qquad k=1,\ldots,d.
\end{equation}
The BAR identity \eqref{eq:bar} holds throughout the common transform domain and its analytic continuation (see Appendix~\ref{app:bar-derivation} for the complete derivation) and couples the $d+1$ unknown holomorphic functions $(\phi_0,\phi_1,\ldots,\phi_d)$ solely through the model parameters $(\mu,\Sigma,R)$.

\subsection{From learned transforms to tails}
Once the interior transform $\phi_0(\theta)$ is learned, restricting it to the complex diagonal gives the Laplace transform of the total workload $S$, namely $\phi_S(s)=\phi_0(s\bm{1})$. Integrating the survival function $p(t)=\Pp_\pi(S>t)$ yields the tail Laplace transform
\begin{equation}\label{eq:tail-lt}
G_S(s)=\int_0^\infty e^{-st}p(t)\,dt=\frac{1-\phi_0(s\bm{1})}{s},\qquad \Re s>0.
\end{equation}
Following \citet{DaiZhang2026}, we parameterize the log transforms, train them to satisfy the BAR and structural constraints, and recover tail probabilities by applying numerical Laplace inversion \citep{AbateWhitt1995}---specifically the Fixed-Talbot algorithm \citep{AbateValko2004,AbateWhitt2006}---to the learned tail transform
\begin{equation}\label{eq:learned-lt}
\widehat{G}_S(s)=\frac{1-\widehat\phi_0(s\bm{1})}{s}.
\end{equation}
Because Fixed-Talbot inversion evaluates $\widehat{G}_S(s)$ at complex quadrature nodes along a deformed contour in the complex plane, accurate tail recovery requires a representation that maintains analyticity and numerical stability across complex arguments, paired with an inversion degree suited to the system dimension.

\section{Holomorphic Cauchy Networks}\label{sec:method}
Following the shared-coordinate log-transform formulation of \citet{DaiZhang2026}, we parameterize the interior and boundary Laplace transforms as
\begin{equation}\label{eq:log-rep}
\widehat\phi_k(\theta)=\exp(f_k(\theta)),\qquad k=0,\ldots,d,
\end{equation}
which stabilizes evaluations across regions where transform magnitudes vary exponentially. HCN constructs $f_k(\theta)$ using native complex Cauchy subnetworks, origin anchors, and boundary-coordinate exclusion. These constructions enforce the prescribed transform values at the origin and boundary-coordinate independence exactly, while ensuring holomorphicity on pole-free domains. The BAR and real-axis constraints are imposed through the training loss.

\subsection{Complex Cauchy representation}
For each hidden unit $r=1,\ldots,H$, we extend the rational Cauchy activation of \citet{LiXiaZhang2025} and \citet{SiEtAl2026} to complex inputs and complex coefficients:
\begin{equation}\label{eq:cauchy}
\sigma_r(z)=\frac{a_r z+b_r}{z^2+s_r^2},\qquad s_r=\operatorname{softplus}(\rho_r)+s_{\min},
\end{equation}
with trainable complex coefficients $a_r,b_r\in\CC$, trainable scale parameter $\rho_r\in\RR$, and fixed floor $s_{\min}>0$. Each scalar subnetwork on $\CC^p$ takes the form
\begin{equation}\label{eq:block}
q_\omega(x)=w_2^{\mathsf T}\sigma(W_1x+c_1)+c_2,\qquad x\in\CC^p,
\end{equation}
where $W_1\in\CC^{H\times p}$, $c_1,w_2\in\CC^H$, and $c_2\in\CC$ are complex weights and biases, and $w_2^{\mathsf T}$ uses the ordinary transpose without complex conjugation. Trainable complex embeddings identify coordinates and boundary faces and remain constant with respect to $\theta$, while positive $s_r\ge s_{\min}$ regulates the locations of the simple poles at $z=\pm\mathrm{i}s_r$.

\begin{proposition}[Structural holomorphicity]\label{prop:holo}
Fix network parameters for which all anchor evaluations are well defined, and let $U\subset\CC^d$ be an open domain on which every Cauchy denominator appearing in $f_0,\ldots,f_d$ is nonzero. Then $f_0,\ldots,f_d$ and their exponentials $\widehat\phi_0,\ldots,\widehat\phi_d$ are holomorphic throughout $U$.
\end{proposition}
\begin{proof}
Complex affine maps and pole-free rational activations $\sigma_r(z)$ are holomorphic in $z$. Composition, finite summation, coordinate restriction, subtraction of well-defined anchors, and complex exponentiation preserve holomorphicity, so $f_k$ and $\widehat\phi_k=\exp(f_k)$ satisfy the Cauchy--Riemann identities on $U$.
\end{proof}
By Proposition~\ref{prop:holo}, HCN satisfies the Cauchy--Riemann equations on every pole-free open set by construction, removing the sampled Cauchy--Riemann penalty $\mathcal L_{\rm CR}$ and the computational graphs required to differentiate that penalty with respect to network parameters.

\subsection{Shared coordinate representations and interactions}
Let $e_j\in\CC^{p_e}$ and $b_k\in\CC^{p_b}$ denote trainable complex embedding vectors for coordinate $j\in\{1,\ldots,d\}$ and boundary face $k\in\{1,\ldots,d\}$, respectively. Separate shared subnetworks $q_{\rm int}$ and $q_{\rm bd}$ give the coordinate representations:
\begin{align}
 f_0^{\rm coord}(\theta)&=\sum_{j=1}^d\bigl[q_{\rm int}(\theta_j,e_j)-q_{\rm int}(0,e_j)\bigr],\label{eq:int-coord}\\
 f_k^{\rm coord}(\theta)&=\log m_k+\sum_{j\ne k}\bigl[q_{\rm bd}(\theta_j,e_j,b_k)-q_{\rm bd}(0,e_j,b_k)\bigr],\qquad k=1,\ldots,d,\label{eq:bd-coord}
\end{align}
where arguments denote concatenation and $m=-R^{-1}\mu$ is the known boundary-mass vector from \eqref{eq:mass}. Anchoring at zero and excluding $\theta_k$ from the $k$th boundary representation imply
\begin{equation}\label{eq:identities}
\widehat\phi_0(0)=1,\qquad\widehat\phi_k(0)=m_k,\qquad\partial_{\theta_k}f_k(\theta)=0,\qquad k=1,\ldots,d.
\end{equation}
The origin anchors enforce the prescribed interior and boundary transform masses, while exclusion of $\theta_k$ enforces normal-coordinate independence for the $k$th boundary representation. These identities hold wherever the network evaluations are well defined and require no additional penalty terms.

The coordinate interior representation is separable: $\exp(f_0^{\rm coord}(\theta))=\prod_{j=1}^d\exp[q_{\rm int}(\theta_j,e_j)-q_{\rm int}(0,e_j)]$. When this is a probability transform, it describes independent coordinates, and the boundary representations likewise impose separability within each face. They suit product-form benchmarks and have parameter count linear in $d$ at fixed widths: $p_e=p_b=16$ and width $H=128$ give $15{,}108+64d$ real scalar parameters (Appendix~\ref{app:parameter-count}).

For dependent transforms, we add anchored interaction subnetworks $g_{\rm int}$ and $g_{\rm bd}$:
\begin{align}
 f_0(\theta)&=f_0^{\rm coord}(\theta)+g_{\rm int}(\theta)-g_{\rm int}(0),\label{eq:global-int}\\
 f_k(\theta)&=f_k^{\rm coord}(\theta)+g_{\rm bd}(\theta_{-k},b_k)-g_{\rm bd}(0,b_k),\qquad k=1,\ldots,d,\label{eq:global-bd}
\end{align}
where $\theta_{-k}\in\CC^{d-1}$ preserves the order of the remaining coordinates and the boundary zero lies in $\CC^{d-1}$. Mixing coordinates before applying a Cauchy activation permits cross-coordinate interactions while retaining \eqref{eq:identities} and Proposition~\ref{prop:holo}. The coordinate and interaction blocks are trained jointly, allowing the architecture to represent nonseparable log transforms while preserving exact origin normalization and boundary-coordinate exclusion. In two dimensions, each boundary face has only one free coordinate, so the 2D Harrison model uses an interior interaction block $g_{\rm int}$ only; the 10D correlated model uses both interaction components (Appendix~\ref{app:parameter-count}).

\subsection{Approximation with shared blocks and exact constraints}\label{sec:approx}
We next establish approximation properties for these architectures, allowing hidden widths to increase while keeping the embedding dimensions fixed. The target class is holomorphic, consistent with transform learning. While polynomial-based approximation is a standard tool for analytic neural networks \citep{ParkWojtowytsch2022}, our construction treats the rational Cauchy activation in \eqref{eq:cauchy} directly and incorporates shared coordinate/boundary embeddings, origin anchors, and boundary-coordinate exclusion.

For $c\in\CC^d$ and radii $r=(r_1,\ldots,r_d)\in(0,\infty)^d$, write $D(c,r)=\{\theta\in\CC^d:|\theta_j-c_j|<r_j,\,j=1,\ldots,d\}$. Fix nested polydiscs
\begin{equation}\label{eq:polydisc}
K=\overline{D(c,r)},\qquad U=D(c,r^+),\qquad |c_j|<r_j<r_j^+<\infty,\qquad j=1,\ldots,d,
\end{equation}
so that $0$ lies in the interior of $K\Subset U$. Define the simultaneous value and first-derivative norm
\[
\|F\|_{1,K}=\max\biggl\{\sup_{\theta\in K}|F(\theta)|,\;\max_{1\le j\le d}\sup_{\theta\in K}|\partial_j F(\theta)|\biggr\}.
\]
Let $F_0,F_1,\ldots,F_d$ be holomorphic on $U$, with $F_0(0)=0$, $F_k(0)=\log m_k$, and $F_k$ independent of $\theta_k$ for $k\ge 1$.

\begin{theorem}[Constrained HCN approximation]\label{thm:approx}
Fix $d\ge 2$, $p_e,p_b\ge 1$, and embeddings $e_j\in\CC^{p_e}$ and $b_k\in\CC^{p_b}$ that are pairwise distinct within each family. For every $\eta>0$, the following statements hold:
\begin{enumerate}\setlength{\itemsep}{2pt}\setlength{\topsep}{3pt}
\item[(i)] \textbf{Coordinate-separable targets.} If the targets are separable,
\[
F_0(\theta)=\sum_{j=1}^d h_j(\theta_j),\qquad F_k(\theta)=\log m_k+\sum_{j\ne k}h_{kj}(\theta_j),\qquad k=1,\ldots,d,
\]
with univariate functions $h_j,h_{kj}$ holomorphic on the corresponding coordinate discs of $U$ and zero at the origin, then the coordinate architecture \eqref{eq:int-coord}--\eqref{eq:bd-coord} admits finite-width blocks satisfying
\begin{equation}\label{eq:approx-bound}
\max_{0\le k\le d}\|f_k-F_k\|_{1,K}<\eta.
\end{equation}
\item[(ii)] \textbf{General dependent targets.} For general targets satisfying the preceding holomorphicity, origin, and boundary-independence conditions, the interaction architecture \eqref{eq:global-int}--\eqref{eq:global-bd} admits finite-width blocks satisfying \eqref{eq:approx-bound}. Its coordinate component may be fixed in advance, provided all its required denominators are nonzero on $U$.
\end{enumerate}
In both cases, the identities \eqref{eq:identities} hold exactly. For any fixed $s_*>s_{\min}$, all constructed blocks can use scale $s_*$ and have every raw denominator bounded below by $\tfrac34 s_*^2$ in modulus on the required inputs from $K$, including anchors.
\end{theorem}

The proof (Appendix~\ref{app:proofs}) first approximates multivariate polynomials by Cauchy blocks with a uniform denominator margin via roots-of-unity filtering, and then uses Lagrange interpolation over projections of the distinct embeddings $e_j$ and $b_k$ to realize all coordinate and boundary slices through shared blocks.

Part (i) establishes representation sufficiency for separable targets despite parameter sharing; part (ii) removes the separability restriction through anchored interactions. Conversely, every coordinate-only interior representation satisfies $\partial_i\partial_j f_0^{\rm coord}(\theta)=0$ for $i\ne j$. Cauchy's integral estimate therefore gives, whenever $F_0$ is holomorphic near $K$,
\begin{equation}\label{eq:obstruction}
\sup_{\theta\in K}\bigl|F_0(\theta)-f_0^{\rm coord}(\theta)\bigr|\ge r_i r_j\bigl|\partial_i\partial_j F_0(c)\bigr|,\qquad i\ne j.
\end{equation}
At $c=0$, for an analytic stationary log transform $F_0=\log\phi_0$ with finite second moments, the mixed derivative is $\partial_i\partial_j F_0(0)=\operatorname{Cov}_\pi(Z_i,Z_j)$. Nonzero cross-covariance thus creates a uniform approximation obstruction for the joint interior log transform on $K$ that increasing coordinate-only width cannot remove, motivating the interaction architecture in \eqref{eq:global-int}--\eqref{eq:global-bd}.

\subsection{Training constraints}\label{sec:constraints}
The training objective follows the BAR learning framework of \citet{DaiZhang2026}. For complex interior samples $\theta^{(n)}$, $n=1,\ldots,N_{\rm BAR}$, a common real shift $\nu_n$ stabilizes the exponential terms:
\begin{equation}\label{eq:kappas}
\kappa_{nk}=\gamma_k(\theta^{(n)})e^{f_k(\theta^{(n)})-\nu_n},\qquad A_n=\kappa_{n0},\qquad B_n=\sum_{k=1}^d\kappa_{nk}.
\end{equation}
We define the normalized relative residual and its weighted mean by
\begin{equation}\label{eq:barloss}
\mathcal E(A,B)=\frac{|A-B|^2}{(|A|+|B|+\varepsilon)^2},\qquad
\mathcal L_{\rm BAR}=\frac{\sum_{n=1}^{N_{\rm BAR}} w_n\mathcal E(A_n,B_n)}{\sum_{n=1}^{N_{\rm BAR}} w_n+\varepsilon_w},
\end{equation}
where $\varepsilon,\varepsilon_w>0$ are stabilizers and the detached weights $w_n\in[0.5,2]$ emphasize larger residuals (Appendix~\ref{app:stable-loss}).

Additional BAR constraints are obtained by evaluating the transforms at auxiliary arguments that isolate a single boundary term. For sampled boundary-batch vectors $y^{(n)}\in\CC^d$, $n=1,\ldots,N_{\rm bd}$, and standard basis vectors $u_k\in\RR^d$, define
\begin{equation}\label{eq:isolated}
v_k=\frac{R^{-\mathsf T}u_k}{(R^{-\mathsf T})_{kk}},\qquad \widetilde\theta^{(n,k)}=y_k^{(n)}v_k,\qquad k=1,\ldots,d.
\end{equation}
Since $R^{\mathsf T}v_k=u_k/(R^{-\mathsf T})_{kk}$ (with $(R^{-\mathsf T})_{kk}=1$ in all benchmarks considered here), all boundary coefficients except the $k$th vanish, reducing the BAR at $\widetilde\theta^{(n,k)}$ to the two-term relation $\gamma_0\phi_0=\gamma_k\phi_k$. We evaluate the pairwise boundary loss directly via the shift-stabilized transform residual:
\begin{equation}\label{eq:bdloss}
\mathcal L_{\rm bd}=\frac1{N_{\rm bd}d}\sum_{k=1}^d\sum_{n=1}^{N_{\rm bd}}\mathcal E\bigl(\widetilde\kappa_0^{(n,k)},\widetilde\kappa_k^{(n,k)}\bigr),
\end{equation}
where each pair shares a common stabilizing shift, avoiding complex logarithms of the BAR coefficients.

At real arguments in their convergence domain, transforms of positive measures are positive and coordinatewise nonincreasing. For real samples $x^{(n)}\in\RR^d$, $n=1,\ldots,N_{\rm real}$, write $J_{nkj}=\partial_j f_k(x^{(n)})$. We encourage real, nonincreasing log transforms through
\begin{align}
\mathcal L_{\rm mono}&=\frac{\sum_{n=1}^{N_{\rm real}}\sum_{k=0}^d\sum_{j=1}^d[\max\{0,\Re J_{nkj}\}]^2}{N_{\rm real}(d+1)d},\label{eq:monoloss}\\
\mathcal L_{\rm imag}&=\frac{\sum_{n=1}^{N_{\rm real}}\sum_{k=0}^d[\Im f_k(x^{(n)})]^2}{N_{\rm real}(d+1)}.\label{eq:imagloss}
\end{align}
The complete training loss is
\begin{equation}\label{eq:full-loss}
\mathcal L=\mathcal L_{\rm BAR}+\mathcal L_{\rm bd}+\mathcal L_{\rm mono}+\mathcal L_{\rm imag},
\end{equation}
with all four weights equal to one in all reported experiments. Moreover, if the target transforms satisfy the BAR and real-axis monotonicity on a compact sampling support covered by Theorem~\ref{thm:approx}, uniform log-transform approximation with $\eta\downarrow 0$ drives all four loss terms in \eqref{eq:full-loss} simultaneously to zero (Appendix~\ref{app:loss-consequences}).

\section{Training under GPU Memory Constraints}\label{sec:training}
While structural holomorphicity completely eliminates the Cauchy--Riemann penalty and its derivative graph, evaluating the BAR residual at each sample still couples $d+1$ transforms, and the real-axis monotonicity penalty requires first-order coordinate derivatives. To make high-dimensional training fast and memory-efficient on a single GPU, we combine closed-form complex input derivatives, sequential loss-branch backpropagation, and dimension-adapted batch schedules.

\subsection{Explicit input derivatives}
Because the rational Cauchy activation \eqref{eq:cauchy} and scalar subnetwork \eqref{eq:block} are elementary holomorphic functions, their complex derivatives admit simple closed-form expressions:
\begin{align}
\sigma'_r(z)&=\frac{a_r(s_r^2-z^2)-2b_rz}{(z^2+s_r^2)^2},\label{eq:cauchyprime}\\
D_xq_\omega(x)&=\bigl(w_2\odot\sigma'(W_1x+c_1)\bigr)^{\mathsf T}W_1,\label{eq:blockprime}
\end{align}
where $\odot$ denotes the elementwise product and no complex conjugation enters the holomorphic chain rule. Coordinate blocks require only the scalar derivative with respect to their univariate transform input $\theta_j$, whereas interaction blocks extract the corresponding columns of $W_1$. For the interior log transform, the $j$th partial derivative is simply
\begin{equation}\label{eq:fprime}
\partial_jf_0(\theta)=\partial_1q_{\rm int}(\theta_j,e_j)+\partial_jg_{\rm int}(\theta),
\end{equation}
where the second term is omitted in the coordinate-separable model, and origin anchors vanish under $\partial_j$ while retaining their dependence on trainable network parameters.

We evaluate input derivatives explicitly only for the real-axis monotonicity loss $\mathcal L_{\rm mono}$; the interior BAR loss $\mathcal L_{\rm BAR}$ and isolated boundary loss $\mathcal L_{\rm bd}$ require forward function evaluations alone. Computing $J_{nkj}=\partial_j f_k(x^{(n)})$ directly via the closed-form expressions \eqref{eq:cauchyprime}--\eqref{eq:fprime} replaces the inner automatic-differentiation pass used to construct input derivatives. The explicit Jacobian expressions remain differentiable with respect to the network parameters $\omega$, and reverse-mode automatic differentiation computes the parameter gradients of the resulting losses.

\subsection{Sequential backpropagation and the sampling budget}
In a $d$-dimensional RBM, evaluating all $d$ boundary log transforms across a batch of $B$ arguments requires $Bd(d-1)$ coordinate--face pair evaluations in addition to origin anchors. Linear parameter growth therefore does not imply linear training-memory growth. To bound individual tensor allocations during the forward pass, we chunk coordinate--face pair evaluations and group boundary faces in the high-dimensional product-form implementation.

To prevent backward-pass intermediate activations from accumulating across all loss terms simultaneously, our 30D and 50D implementations further evaluate and backpropagate the interior BAR branch ($\mathcal L_{\rm BAR}$), the isolated boundary branch ($\mathcal L_{\rm bd}$), and the combined real-axis branch ($\mathcal L_{\rm mono}+\mathcal L_{\rm imag}$) in sequence. Each branch executes its backward pass, accumulates parameter gradients, and releases its computation graph before the next branch is constructed, with anchors recomputed in fresh graphs:
\begin{equation}\label{eq:accum}
\nabla_\omega\mathcal L=\nabla_\omega\mathcal L_{\rm BAR}+\nabla_\omega\mathcal L_{\rm bd}+\nabla_\omega(\mathcal L_{\rm mono}+\mathcal L_{\rm imag}).
\end{equation}
For fixed parameters $\omega$, samples, and detached BAR weights $w_n$, sequential accumulation yields the same summed-loss gradient as a combined backward pass up to floating-point arithmetic, while reducing the simultaneous storage of branch-specific backward intermediates. Global gradient-norm clipping and a single Adam update are then applied once all three branches have accumulated their gradients (Algorithm~1).

\begin{center}
\begin{minipage}{0.96\linewidth}\small
\noindent\textbf{Algorithm 1. One optimizer update with sequential loss backpropagation}\par\vspace{3pt}
\begin{enumerate}\setlength{\itemsep}{1pt}\setlength{\topsep}{1pt}\setlength{\parskip}{0pt}
\item Clear gradients and sample fresh BAR, boundary, and real-axis batches.
\item Evaluate the BAR loss $\mathcal L_{\rm BAR}$; backpropagate and release its computation graph.
\item Evaluate the boundary loss $\mathcal L_{\rm bd}$; accumulate its gradients and release its graph.
\item Evaluate real-axis values and explicit Jacobians \eqref{eq:cauchyprime}--\eqref{eq:fprime}; backpropagate $\mathcal L_{\rm mono}+\mathcal L_{\rm imag}$ and release the graph.
\item Check gradient finiteness, apply global gradient-norm clipping with threshold $1.0$ to the accumulated gradient, and take one Adam update.
\end{enumerate}
\end{minipage}
\end{center}

The 10D implementation additionally accumulates boundary faces separately within Step~3, retaining the $1/d$ factor in \eqref{eq:bdloss}. The 2D and 20D implementations use a combined backward call. These are graph-scheduling choices rather than changes to the training objective.

We count freshly sampled transform-argument vectors as base inputs. With $U$ optimizer updates, the cumulative sampling budget is
\begin{equation}\label{eq:budget}
N_{\rm train}=U(B_{\rm BAR}+B_{\rm bd}+B_{\rm real}),
\end{equation}
where a boundary base input supplies $d$ isolated arguments \eqref{eq:isolated} and a real-axis base input supplies two losses; derived arguments and origin anchors are excluded from $N_{\rm train}$. As dimension increases, Table~\ref{tab:batches} reduces per-update batch sizes and increases the number of optimizer updates $U$ to maintain a per-run budget of 64 million sampled base inputs in the high-dimensional experiments.

\begin{table}[htbp]\centering\small
\caption{HCN batches and fixed training endpoints.}\label{tab:batches}
\begin{tabular}{lrrrrr}\toprule
Benchmark & Updates ($U$) & $B_{\rm BAR}$ & $B_{\rm bd}$ & $B_{\rm real}$ & Base inputs ($N_{\rm train}$)\\\midrule
2D, both models & 15,000 & 3,416 & 427 & 427 & 64.05M\\
10D non-product & 25,000 & 2,048 & 256 & 256 & 64M\\
20D product & 25,000 & 2,048 & 256 & 256 & 64M\\
30D product & 50,000 & 1,024 & 128 & 128 & 64M\\
50D product & 100,000 & 512 & 64 & 64 & 64M\\\bottomrule
\end{tabular}
\par\vspace{3pt}\begin{minipage}{.91\linewidth}\footnotesize $\mathrm M=10^6$ base inputs per run. Updates count optimizer steps, rather than backward calls.\end{minipage}
\end{table}

\subsection{Sampling and optimization}
Each update resamples arguments. The BAR batch mixes approximately 50\% general complex inputs, 20\% complex-diagonal inputs, 20\% near-diagonal inputs, and a remainder on the nonnegative real diagonal. Diagonal coverage reflects the query $\phi_0(s\bm{1})$ used for sum tails, while off-diagonal inputs constrain the multivariate BAR. The nested-box sampler and integer rounding are documented in Appendix~\ref{app:implementation}.

Complex samples are drawn from a bounded domain $\Theta=[\underline{\theta}^{\rm Re},\overline{\theta}^{\rm Re}]^d\times[-\overline{\theta}^{\rm Im},\overline{\theta}^{\rm Im}]^d$ following the two-stage nested-box sampling scheme of \citet{DaiZhang2026}. Boundary base vectors are drawn from the same complex sampler and mapped to the isolated arguments in \eqref{eq:isolated}, while real-axis batches use the corresponding real-part projection of the nested-box sampler.

All HCN training runs use \texttt{complex128} arithmetic and Adam with learning rate $5\times10^{-5}$ for the first 80\% of updates and $10^{-5}$ thereafter, retaining optimizer state across the schedule transition. Gradient clipping follows gradient accumulation, and preflight denominator and finiteness checks verify evaluated inputs (Appendix~\ref{app:pole-checks}). All reported evaluations in Section~\ref{sec:experiments} use the fixed endpoints in Table~\ref{tab:batches}.

\section{Numerical Experiments}\label{sec:experiments}
We evaluate resource requirements, tail accuracy, and extension beyond product form across 2- to 50-dimensional benchmarks. We also examine numerical inversion degree, complex activation choice, and late-training stability in the 50D run. Throughout all experiments, networks are trained solely from the RBM parameters $(\mu,\Sigma,R)$; analytical and simulation references enter post-training evaluation only.

\subsection{Benchmarks and evaluation protocol}
We evaluate the product-form benchmark family used by \citet{DaiZhang2026} at $d\in\{2,20,30,50\}$, with the 50-dimensional case introduced here. The instances use $\mu=-\bm{1}$, $R_{ii}=1$, $R_{i,i-1}=-1$, and $\Sigma=R+R^{\mathsf T}$, with unspecified entries zero. Under this skew-symmetric structure, the stationary coordinates are independent $\operatorname{Exp}(i)$ variables \citep{HarrisonWilliams1987}, giving exact sum-workload transform and tail probability
\begin{equation}\label{eq:product-true}
\phi_S(s)=\prod_{i=1}^d\frac{i}{i+s},\qquad p_d(t)=1-(1-e^{-t})^d,\qquad t\ge0.
\end{equation}
The two-dimensional Harrison non-product benchmark \citep{DaiHarrison1992,DaiZhang2026} has parameters
\begin{equation}\label{eq:harrison}
\mu=\begin{pmatrix}-1\\0\end{pmatrix},\qquad \Sigma=I_2,\qquad R=\begin{pmatrix}1&0\\-1&1\end{pmatrix}.
\end{equation}
Its reference tail probabilities are obtained by one-dimensional polar quadrature of the known stationary density (Appendix~\ref{app:reference-formulas}). The exploratory ten-dimensional correlated non-product model is specified in Section~\ref{sec:10d}.

We evaluate the fixed endpoints in Table~\ref{tab:batches} using a common random seed (12345), and include three seeds (12345, 22345, and 32345) at 20D to summarize initialization variability via the sample mean and standard deviation of per-run errors. No minimum-error checkpoint is substituted for a fixed endpoint. Fixed-Talbot inversion uses degree $M=5$ in 2D and $M=8$ in higher dimensions; at 30D, we also report $M=5$ for the degree-matched comparison with published NN4MGF results. Changing the inversion degree does not retrain the network.

Following the probability-clipping convention in the released high-dimensional NN4MGF evaluation code, all finite real inversion outputs are projected onto $[0,1]$ via $\widehat p_M(t)=\min\bigl\{1,\allowbreak\max\{0,\allowbreak T_M[\widehat G_S](t)\}\bigr\}$. Against positive reference probabilities, mean relative error (MRE, in percent) is
\begin{equation}\label{eq:mre}
\MRE_{1:K}=\frac{100}{K}\sum_{t=1}^{K}\frac{|\widehat p_M(t)-p_{\rm ref}(t)|}{p_{\rm ref}(t)}.
\end{equation}
We use $\MRE_{1:8}$ as the primary accuracy metric for the product-form benchmarks and the exploratory 10D example, and additionally report $\MRE_{1:5}$ and $\MRE_{1:10}$ for the product-form benchmarks to characterize shorter and extended tail ranges. The 2D Harrison benchmark uses $\MRE_{1:5}$ to match the evaluation range of \citet{DaiZhang2026}. The 2D and 30D NN4MGF MREs are computed from the pointwise tail predictions released in the official code repository of \citet{DaiZhang2026}, while 20D uses archived Degree-5 and Degree-8 evaluations of the released NN4MGF checkpoint under their respective training budgets.

\subsection{Computational performance and dimensional scaling}\label{sec:performance}
Both 30D configurations are profiled on the same NVIDIA H100 NVL GPU, with 10 warmup updates followed by 200 measured updates. CUDA synchronization barriers bracket update timing, and peak-memory statistics are reset after warmup. Timings include in-loop sampling, forward evaluation, and backpropagation, excluding model construction and final output generation.

\begin{table}[!htbp]
\centering\small
\caption{Thirty-dimensional training profiles on the same NVIDIA H100 NVL.}
\label{tab:profile}
\setlength{\tabcolsep}{4pt}
\begin{tabular}{lrr}
\toprule
Quantity & HCN & NN4MGF \\
\midrule
Real scalar parameters & 17,028 & 43,080 \\
Base inputs / update & 1,280 & 32,768 \\
Measured updates (after 10 warmup) & 200 & 200 \\
Time for 200 updates (s) & 33.023 & 320.209 \\
Mean time / update (s) & 0.165113 & 1.601046 \\
Peak allocated memory (GiB) & 7.805 & 77.815 \\
Peak reserved memory (GiB) & 8.414 & 81.428 \\
Updates in full schedule & 50,000 & 300,000 \\
Base inputs in full schedule & 64M & 9.8304B \\
Projected full-schedule update time (h) & 2.293 & 133.420 \\
\bottomrule
\end{tabular}
\par\vspace{3pt}\begin{minipage}{\linewidth}\footnotesize Both use double precision: complex128 for HCN and float64 for NN4MGF. Times and memory peaks are measured over the profile; full-schedule times are projections. GiB denotes $2^{30}$ bytes. The methods use their respective batches and architectures.\end{minipage}
\end{table}

As shown in Table~\ref{tab:profile}, mean per-update times are 0.165113~s for HCN and 1.601046~s for NN4MGF, giving a \textbf{9.7-fold reduction in mean update time} (ratio $9.697$). Peak allocated GPU memory is 7.805~GiB for HCN versus 77.815~GiB for NN4MGF, an \textbf{approximately tenfold reduction} (ratio $9.970$). At 30D, the tested NN4MGF configuration reserves 81.428~GiB of H100 memory, reflecting the memory cost of retaining first-order derivative graphs across large batches for the derivative-based penalties. By removing $\mathcal L_{\rm CR}$ structurally and combining explicit real-axis Jacobians with sequential branch backpropagation, HCN keeps peak reserved memory under 8.5~GiB.

Multiplying the measured mean update times by the respective scheduled update counts ($50{,}000$ for HCN and $300{,}000$ for NN4MGF) gives projected full-schedule update times of 2.293 and 133.420 hours:
\begin{equation}\label{eq:profile-ratio}
\frac{\widehat T_{\NN}}{\widehat T_{\HCN}}
=\frac{1.60104552}{0.16511285}\,\frac{300{,}000}{50{,}000}
\approx9.697\times6=58.18.
\end{equation}
The \textbf{58.2-fold projected full-schedule update-time ratio} combines the measured 9.7-fold per-update ratio with the sixfold difference in scheduled update counts. This projection extrapolates the short-profile measurements to the respective full training schedules rather than measuring completed training runs or time to a common accuracy target, and compares the complete tested configurations—including their architectures, batch sizes, derivative calculations, and graph schedules—rather than isolating the contribution of CR penalty removal.

The tested NN4MGF code path samples two batches of $16{,}384$ per update, giving $300{,}000\times2\times16{,}384=9.8304$ billion base inputs over its full schedule. HCN's 64M cumulative budget represents $0.651\%$ of that count ($1.302\%$ of the paper-reported $300{,}000\times16{,}384$ single-batch count; Appendix~\ref{app:resource-accounting}). Separately, this reduced memory footprint allows all HCN training experiments in this study to be completed on a single GPU, scaling up to the 50D benchmark ($100{,}000$ updates, $18{,}308$ real scalar parameters, and 64M base inputs) on a single NVIDIA GeForce RTX 4080 Laptop GPU.

\subsection{Tail accuracy and the 50D extension}
Table~\ref{tab:accuracy} compares tail errors at the fixed training endpoints. Across all degree-matched comparisons over the primary evaluation ranges, HCN attains lower reported MRE than NN4MGF. The error reductions are substantial on the two-dimensional benchmarks—especially on the non-product Harrison model—and modest on the 20D and 30D product-form benchmarks, where HCN achieves competitive accuracy using less than $1\%$ of the baseline sampling budget while extending computation to 50 dimensions.

\begin{table}[!htbp]
\centering\small
\caption{Tail accuracy at fixed training endpoints (MRE, percent).}
\label{tab:accuracy}
\setlength{\tabcolsep}{4pt}
\begin{tabular}{lllrrr}
\toprule
Benchmark & Method & $M$ & $1{:}5$ & $1{:}8$ & $1{:}10$ \\
\midrule
2D product & HCN & 5 & $2.147$ & $2.213$ & $9.737$ \\
2D product & NN4MGF (published) & 5 & $2.448$ & $8.050$ & $13.915$ \\
2D Harrison & HCN & 5 & $1.422$ & --- & --- \\
2D Harrison & NN4MGF (published) & 5 & $6.684$ & --- & --- \\
20D product & HCN & 8 & $1.441 \pm 1.322$ & $7.632 \pm 1.565$ & $12.264 \pm 3.678$ \\
20D product & NN4MGF (archived) & 8 & $2.802$ & $8.148$ & $18.908$ \\
30D product & HCN (comparison) & 5 & $3.429$ & $10.107$ & $19.084$ \\
30D product & NN4MGF (published) & 5 & $8.353$ & $10.476$ & $21.507$ \\
30D product & HCN (primary) & 8 & $1.683$ & $6.546$ & $8.872$ \\
50D product & HCN (extension) & 8 & $1.495$ & $6.367$ & $8.011$ \\
\bottomrule
\end{tabular}
\par\vspace{3pt}\begin{minipage}{\linewidth}\footnotesize The 20D HCN entries are the mean $\pm$ sample SD over three seeds; SD is in percentage points. Other entries are single runs or single baseline records. The same 30D HCN weights are used at both degrees. Dashes indicate unavailable reporting ranges.\end{minipage}
\end{table}

At 2D (Figure~\ref{fig:2d}), the product-form $1{:}8$ errors are $2.213\%$ for HCN and $8.050\%$ for NN4MGF ($9.737\%$ versus $13.915\%$ on $1{:}10$). On the non-product 2D Harrison benchmark, HCN with the interior interaction branch $g_{\rm int}$ reduces $\MRE_{1:5}$ to \textbf{1.422\%} compared with \textbf{6.684\%} for NN4MGF. Here the HCN Harrison architecture matches the HCN 2D product-form model's parameter budget ($15{,}236$ real scalar parameters and 64.05M inputs) by adjusting subnetwork widths to accommodate the interior interaction branch (Appendix~\ref{app:parameter-count}). This improvement illustrates the effectiveness of the anchored interaction representation on a non-product RBM with dependent coordinates.

\begin{figure}[htbp]\centering
\includegraphics[width=.97\linewidth]{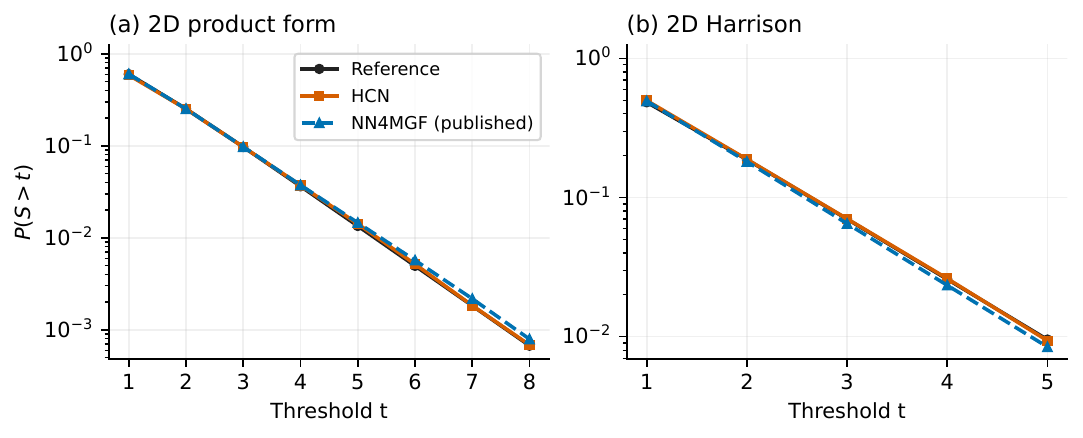}
\caption{2D tails at Degree 5. HCN uses the fixed 15,000-update endpoints; NN4MGF uses published values from \citet{DaiZhang2026}. References are analytical in (a) and density-based in (b).}\label{fig:2d}
\end{figure}

In the high-dimensional product-form benchmarks (Figure~\ref{fig:product}):
\begin{itemize}\setlength{\itemsep}{2pt}\setlength{\topsep}{3pt}
\item \textbf{20D benchmark:} At Degree~8, HCN achieves a three-seed mean $\MRE_{1:8}$ of $7.632\%$ (SD $1.565$ percentage points; individual runs $9.003\%$, $7.966\%$, and $5.927\%$), compared with $8.148\%$ for NN4MGF. On $1{:}5$ and $1{:}10$, HCN's mean MREs are $1.441\%\pm 1.322\%$ and $12.264\%\pm 3.678\%$, compared with $2.802\%$ and $18.908\%$ for NN4MGF.
\item \textbf{30D benchmark:} Under the degree-matched $M=5$ rule, $\MRE_{1:8}$ is $10.107\%$ for HCN versus $10.476\%$ for NN4MGF ($3.429\%$ versus $8.353\%$ on $1{:}5$, and $19.084\%$ versus $21.507\%$ on $1{:}10$). Evaluating the same 30D HCN weights with the primary $M=8$ rule reduces $\MRE_{1:5}$, $\MRE_{1:8}$, and $\MRE_{1:10}$ to $1.683\%$, $6.546\%$, and $8.872\%$, respectively.
\item \textbf{50D extension:} Completed on a single laptop GPU, the 50D HCN run achieves $\MRE_{1:5}$ of \textbf{1.495\%}, $\MRE_{1:8}$ of \textbf{6.367\%}, and $\MRE_{1:10}$ of \textbf{8.011\%}, closely tracking the analytical survival curve across the evaluation range (Figure~\ref{fig:product}c).
\end{itemize}

\begin{figure}[htbp]\centering
\includegraphics[width=.97\linewidth]{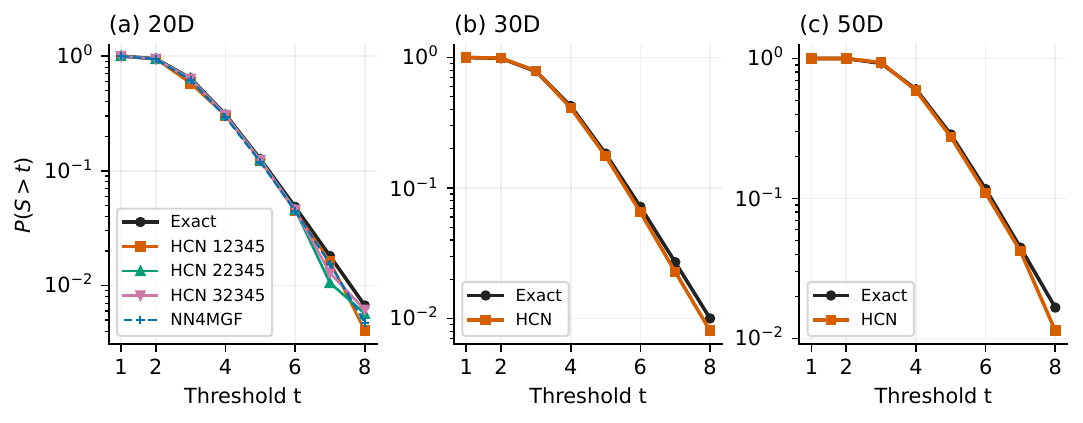}
\caption{Product-form tails at Degree 8. Panel (a) shows the three 20D HCN runs and the archived NN4MGF checkpoint. Panels (b) and (c) show the fixed endpoints at 30D and 50D.}\label{fig:product}
\end{figure}

\subsection{Inversion and activation diagnostics}\label{sec:inversion}
When initially evaluating our models under the baseline Degree $M=5$ rule, we observed that HCN achieved lower error than NN4MGF in 2D and 30D, yet showed slightly higher error on segments of the 20D benchmark (for instance, on $1{:}10$ under $M=5$, HCN's three-seed mean MRE was $19.347\%$ versus $16.801\%$ for NN4MGF). This discrepancy led us to ask whether the intrinsic truncation error of the Fixed-Talbot algorithm at $M=5$ was canceling with NN4MGF's approximation error or reinforcing HCN's error. Writing $p_M^{\star}(t)=\min\{1,\max\{0,T_M[G_S](t)\}\}$ for the clipped output obtained by applying Degree-$M$ Fixed-Talbot inversion and the same $[0,1]$ projection rule to the exact transform $G_S$, the clipped prediction error decomposes into two signed components:
\begin{equation}\label{eq:decomposition}
\widehat p_M(t)-p(t)=\underbrace{\bigl(p_M^{\star}(t)-p(t)\bigr)}_{e_{{\rm inv},M}(t)}+\underbrace{\bigl(\widehat p_M(t)-p_M^{\star}(t)\bigr)}_{e_{{\rm net},M}(t)}.
\end{equation}
These two signed components represent intrinsic clipped inversion error and the net effect of substituting the learned transform into the clipped inversion pipeline, respectively, and can partially cancel or reinforce each other when summed.

\begin{table}[!htbp]
\centering\small
\caption{Fixed-Talbot inversion of the exact product-form transform: Degree 5 versus Degree 8 (MRE, percent).}
\label{tab:exact-inversion}
\setlength{\tabcolsep}{6pt}
\begin{tabular}{lrrrrrr}
\toprule
 & \multicolumn{2}{c}{$t=1{:}5$} & \multicolumn{2}{c}{$t=1{:}8$} & \multicolumn{2}{c}{$t=1{:}10$} \\
\cmidrule(lr){2-3}\cmidrule(lr){4-5}\cmidrule(lr){6-7}
Dimension & $M=5$ & $M=8$ & $M=5$ & $M=8$ & $M=5$ & $M=8$ \\
\midrule
20D & 3.391 & \textbf{0.073} & 4.044 & \textbf{0.122} & 9.430 & \textbf{0.332} \\
30D & 3.218 & \textbf{0.152} & 3.955 & \textbf{0.178} & 8.397 & \textbf{0.375} \\
50D & 34.866 & \textbf{0.507} & 23.898 & \textbf{0.414} & 23.352 & \textbf{0.563} \\
\bottomrule
\end{tabular}
\par\vspace{3pt}\begin{minipage}{\linewidth}\footnotesize Only the exact analytical transform is inverted; no neural network approximation enters this table. Bold entries indicate the lower error in each degree pair.\end{minipage}
\end{table}

To test this hypothesis independently of neural approximation, we applied the Fixed-Talbot procedure directly to the exact product-form transforms in \eqref{eq:product-true}. As shown in Table~\ref{tab:exact-inversion}, Degree $M=5$ indeed incurs substantial intrinsic inversion error in higher dimensions ($4.044\%$ at 20D, $3.955\%$ at 30D, and $23.898\%$ at 50D on $1{:}8$), and at 20D this $M=5$ quadrature bias partially cancels NN4MGF's transform error, making the Degree-5 metric favorable to NN4MGF (Appendix~\ref{app:signed-errors}). Increasing the Talbot degree to $M=8$ reduces the exact-transform inversion error on $1{:}8$ to \textbf{0.122\%} at 20D, \textbf{0.178\%} at 30D, and \textbf{0.414\%} at 50D, while keeping exact-transform MRE below $0.57\%$ on $1{:}5$ and $1{:}10$ across all three dimensions. We therefore adopt Degree~8 as our primary high-dimensional evaluation metric. Evaluating the exact same 20D trained weights under Degree~8---where intrinsic inversion bias is largely removed---HCN achieves lower mean MRE than NN4MGF across all evaluation ranges ($1.441\%$ versus $2.802\%$ on $1{:}5$, $7.632\%$ versus $8.148\%$ on $1{:}8$, and $12.264\%$ versus $18.908\%$ on $1{:}10$; Appendix~\ref{app:degree-sensitivity}). Beyond motivating our evaluation protocol, this diagnostic highlights an important practical lesson for BAR-based transform learning: in high dimensions, sufficient inversion degree is needed to prevent quadrature errors from masking or distorting network performance.

\begin{table}[!htbp]
\centering\small
\caption{Activation comparison at 20D and Degree 8: MRE on $1{:}8$ (percent).}
\label{tab:activation}
\setlength{\tabcolsep}{4pt}
\begin{tabular}{lrr}
\toprule
Seed & Cauchy MRE (\%) & Complex-tanh MRE (\%) \\
\midrule
12345 & 9.003 & 32.990 \\
22345 & 7.966 & 19.823 \\
32345 & 5.927 & 13.502 \\
Mean $\pm$ SD & $7.632 \pm 1.565$ & $22.105 \pm 9.942$ \\
\bottomrule
\end{tabular}
\par\vspace{3pt}\begin{minipage}{\linewidth}\footnotesize All three seeds use the identical coordinate architecture, 64M input budget, and fixed 25,000-update endpoint.\end{minipage}
\end{table}

To examine activation choice within the complex architecture, the 20D complex-tanh control retains the exact same coordinate architecture, widths, seeds, 64M input budget, and 25,000-update endpoint. Both activations are holomorphic away from their poles and omit CR penalties, with Cauchy's trainable activation parameters giving $16{,}388$ real scalars versus $15{,}108$ for complex-tanh. As shown in Table~\ref{tab:activation}, Cauchy achieves lower $\MRE_{1:8}$ on every seed at Degree~8, with a three-seed mean of \textbf{7.632\% $\pm$ 1.565\%} compared with \textbf{22.105\% $\pm$ 9.942\%} for complex-tanh.

\subsection{Late-training stability in 50 dimensions}\label{sec:late-training}
To examine whether the 50D tail accuracy is sustained across late-stage updates rather than isolated to a single checkpoint, we evaluate five archived checkpoints along the 50D trajectory from update $92{,}000$ to $100{,}000$ in increments of $2{,}000$ updates under the fixed Degree-8 protocol.

\begin{table}[!htbp]
\centering\small
\caption{Late-training tail accuracy along the 50D trajectory.}
\label{tab:late-training}
\setlength{\tabcolsep}{4pt}
\begin{tabular}{rrl}
\toprule
Update & $1{:}8$ MRE (\%) & Role \\
\midrule
92,000 & 6.650 & Diagnostic checkpoint \\
94,000 & 5.299 & Diagnostic checkpoint \\
96,000 & 6.845 & Diagnostic checkpoint \\
98,000 & 7.155 & Diagnostic checkpoint \\
100,000 & 6.367 & Reported endpoint \\
\bottomrule
\end{tabular}
\par\vspace{3pt}\begin{minipage}{\linewidth}\footnotesize All checkpoints are evaluated at Degree 8. The main accuracy comparison uses the fixed 100,000-update endpoint.\end{minipage}
\end{table}

Table~\ref{tab:late-training} shows $\MRE_{1:8}$ values between $5.299\%$ (at update $94{,}000$) and $7.155\%$ (at update $98{,}000$), with $6.367\%$ at the prescribed $100{,}000$-update endpoint. Errors fluctuate nonmonotonically within a narrow single-digit range across all five checkpoints, providing evidence of late-training stability along this trajectory under the fixed evaluation rule (see Appendix~\ref{app:late-training} for the corresponding $1{:}5$ interval).

\subsection{An exploratory 10D non-product example}\label{sec:10d}
Beyond two dimensions, non-product RBMs lack closed-form joint stationary densities or Laplace transforms, so evaluating tail accuracy requires steady-state simulation \citep{BlanchetEtAl2021}. With general oblique reflection, simultaneously controlling reflection discretization bias across $d$ faces, burn-in to stationarity, and relative Monte Carlo variance in the tail becomes computationally demanding as dimension grows, and within our computational budget we do not have a certified simulation ground truth in 30 or 50 dimensions. We therefore conduct our multidimensional non-product evaluation at $d=10$ under orthogonal reflection ($R=I_{10}$) with correlated Brownian input, a setting that introduces cross-coordinate dependence while permitting both an exact analytical first-moment cross-check and a reproducible $1{,}000{,}000$-path adaptive Brownian-bridge tail reference.

Specifically, to introduce cross-coordinate dependence while retaining known stationary marginals, consider the 10-dimensional orthogonally reflected RBM
\begin{equation}\label{eq:10d-model}
d=10,\qquad R=I_{10},\qquad \mu=-\bm{1},\qquad \Sigma_{ij}=\frac{2\rho^{|i-j|}}{\sqrt{ij}},\qquad\rho=0.3.
\end{equation}
Here $\Sigma$ is a positive diagonal rescaling of an AR(1) correlation matrix. Because $R=I_{10}$, each coordinate $Z_i$ is marginally a one-dimensional RBM with drift $-1$ and variance rate $2/i$, so $Z_i\sim\operatorname{Exp}(i)$ in stationarity. However, nonzero off-diagonal covariances violate the product-form skew-symmetry condition for $R=I_{10}$ \citep{HarrisonWilliams1987}. Consequently,
\begin{equation}\label{eq:independence}
p_{\rm ind}(t)=1-(1-e^{-t})^{10}
\end{equation}
serves as an independent-marginal comparator rather than the correlated system's true sum-tail reference.

\begin{figure}[!htb]\centering
\includegraphics[width=.80\linewidth]{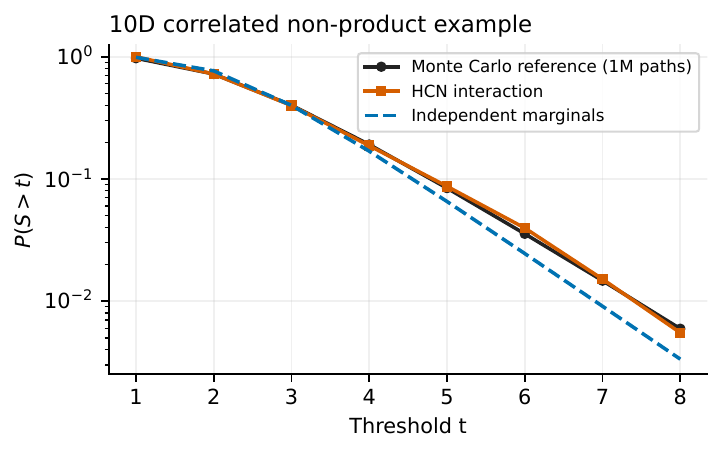}
\caption{Exploratory 10D non-product tails at Degree 8: HCN interaction model compared with the $1{,}000{,}000$-path Monte Carlo reference and independent-marginal tail.}\label{fig:10d}
\end{figure}

We evaluate the HCN interaction model (incorporating both $g_{\rm int}$ and $g_{\rm bd}$) at 25,000 updates and Degree~8. As an independent analytical diagnostic of the learned interior transform at the origin, evaluating $\widehat G_S(0)=-\sum_{j=1}^{10}\partial_j f_0(0)$ yields a learned mean total workload of $2.924299$, which differs from the exact analytical first moment $\E_\pi[S]=H_{10}=\sum_{j=1}^{10}1/j\approx 2.928968$ by only \textbf{0.159\%} (Appendix~\ref{app:10d-reference}). To construct a reproducible sum-tail reference, we exploit the coordinatewise running-maximum representation $Z\stackrel{d}{=}(\sup_{s\ge 0}(B_1(s)-s),\ldots,\sup_{s\ge 0}(B_{10}(s)-s))$ under $R=I_{10}$ and simulate $1{,}000{,}000$ independent correlated paths on horizon $T=64$ using adaptive correlated Brownian-bridge refinement (unresolved sum-bracket tolerance $0.005$, per-path bridge missed-extremum risk $10^{-7}$, and analytical finite-horizon tail bound $q_T\approx 1.54\times 10^{-8}$; see Appendix~\ref{app:10d-reference}). Evaluated against this reference, HCN achieves an $\MRE_{1:8}$ of \textbf{3.829\%} (conservative simultaneous 95\% MRE interval $[2.539\%, 5.142\%]$), compared with \textbf{19.386\%} ($[18.504\%, 20.267\%]$) for the independent-marginal comparator $p_{\rm ind}(t)$. At $t=8$, the simulation reference, HCN prediction, and independent-marginal probabilities are $0.0059435$ (95\% CI $[0.005781, 0.006109]$), $0.005512$, and $0.003350$, respectively (Figure~\ref{fig:10d}). Relative to the independent-marginal comparator, this result provides clear evidence that the anchored interaction representation captures cross-coordinate dependence effects on the sum tail in a multidimensional non-product setting.

\section{Discussion}\label{sec:discussion}
\textbf{Structural holomorphy and memory-controlled computation.} On pole-free domains, native complex Cauchy subnetworks are holomorphic by construction, eliminating sampled Cauchy--Riemann (CR) penalties and the first-order input-derivative graphs they require across complex collocation points. In a $d$-dimensional RBM, however, replacing the activation function alone is not sufficient to resolve the high-dimensional memory bottleneck: each parameter update still evaluates BAR and real-axis monotonicity constraints across one interior transform and $d$ boundary transforms, so practical memory demand also depends on derivative implementation and graph scheduling. Pairing native complex Cauchy representations with closed-form real-axis derivatives, coordinate--face chunking, and sequential three-branch backpropagation yields the runtime and memory reductions observed in the 30D H100 profile. Because HCN and NN4MGF differ in their architectures, batch sizes, scheduled update counts, derivative calculations, and graph schedules, those 30D ratios characterize the complete configured pipelines jointly rather than isolating CR penalty removal alone. In practice, combining this memory-controlled implementation with dimension-dependent batch schedules and the 64M base-input budget allows all tested HCN benchmarks to run on a single GPU and enables the 50D product-form run on a single laptop GPU; across the five recorded checkpoints from $92{,}000$ to $100{,}000$ updates in Section~\ref{sec:late-training}, $\MRE_{1:8}$ remains between $5.299\%$ and $7.155\%$, with the fixed $100{,}000$-update endpoint ($6.367\%$) reported as the primary result.

\textbf{Constrained approximation and non-product dependence.} Theorem~\ref{thm:approx} characterizes fixed-dimension representation capacity under holomorphic-log assumptions: parameter-shared additive coordinate blocks simultaneously approximate function values and first derivatives of separable log transforms on compact polydiscs, while anchored interaction subnetworks approximate general holomorphic targets satisfying exact origin normalization and boundary-coordinate independence. For the coordinate-only architecture used in the product-form benchmarks, sharing subnetwork weights across coordinates and faces avoids replicating full networks for each coordinate or boundary face, so at fixed width only the coordinate and face embedding parameters grow linearly with dimension ($P(d)=15{,}108+64d$ in Appendix~\ref{app:parameter-count}). Conversely, the mixed-derivative lower bound \eqref{eq:obstruction} shows why coordinate-only representations are structurally limited on correlated systems: whenever stationary coordinates have nonzero cross-covariance, the true joint log transform has nonzero mixed partial derivatives at the origin, whereas any purely additive representation has identically zero mixed partials and therefore incurs an irreducible uniform approximation error on the joint log transform over compact polydiscs that increasing width alone cannot remove. Although \eqref{eq:obstruction} bounds uniform log-transform error rather than directly bounding sum-tail MRE after numerical inversion, the empirical improvement on the 2D Harrison benchmark ($1.422\%$ versus $6.684\%$ $\MRE_{1:5}$) and the results on the exploratory 10D correlated RBM are consistent with this structural distinction. In the 10D correlated RBM study ($3.829\%$ versus $19.386\%$ $\MRE_{1:8}$ against the $1{,}000{,}000$-path adaptive Brownian-bridge reference), orthogonal reflection ($R=I_{10}$) enables exact marginal cross-checks and running-maximum simulation bounds while isolating the effect of Brownian covariance; meanwhile, the $0.159\%$ relative discrepancy between the learned origin derivative sum and the analytical mean workload $H_{10}$ provides an independent first-moment consistency check on the interior transform at the origin, even though matching the sum mean alone does not certify cross-coordinate covariance.

\textbf{Separating transform approximation from numerical inversion.} In BAR-based solvers, the neural network learns complex log-Laplace transforms, and tail probabilities are subsequently computed via Fixed-Talbot inversion. Accordingly, under the common $[0,1]$ probability-clipping rule, \eqref{eq:decomposition} decomposes the final clipped tail error $\widehat p_M - p$ into two signed components: the intrinsic clipped inversion error $e_{{\rm inv},M}=p_M^\star - p$ incurred by applying Talbot quadrature to the exact transform, and the net pipeline approximation error $e_{{\rm net},M}=\widehat p_M - p_M^\star$ from replacing the exact transform with the learned network at the complex Talbot nodes. Because both components are signed, they can partially cancel or reinforce each other across thresholds. As a result, even when increasing the Talbot degree substantially reduces exact-transform inversion error---which motivates our choice of Degree~8 in 20 to 50 dimensions---changing the complex evaluation nodes and weights can occasionally expose or amplify network errors that were previously masked by sign cancellation (Appendix~\ref{app:degree-sensitivity}). Tail accuracy should therefore be evaluated for a specific pair of learned transform and inversion rule: where analytical transforms are available, exact-transform diagnostics should be used first to select an appropriate inversion degree, after which learned models are compared at fixed training endpoints under explicit evaluation rules.

\textbf{Scope and future directions.} Several extensions remain open. First, Theorem~\ref{thm:approx} establishes the existence of pole-free approximations at finite but unrestricted width on compact polydiscs under a nonvanishing-transform assumption; it does not by itself guarantee optimization convergence at a fixed trained width, nor do structural holomorphicity and finitely many sampled BAR and real-axis constraints certify that a learned transform corresponds to a globally valid probability measure. Second, because the trainable Cauchy scale parameters $s_r$ govern complex pole locations, relating learned pole distributions directly to the geometry of the Fixed-Talbot contour could strengthen the theoretical foundation for deep-tail inversion. Finally, developing steady-state simulation references with certified error bounds for larger correlated RBMs and conducting single-factor resource ablations will further clarify the scope of holomorphic neural solvers and the individual contributions of their design components.

\section{Conclusion}\label{sec:conclusion}
This paper develops Holomorphic Cauchy Networks (HCNs) for stationary RBM tail computation within NN4MGF's BAR-learning and shared-coordinate framework. By embedding rational Cauchy activations into native complex-valued subnetworks, HCN enforces holomorphicity on pole-free domains by construction without sampled Cauchy--Riemann penalties, while exact origin anchors and boundary-coordinate exclusion enforce normalization and face-independence identities algebraically. Under holomorphic-log assumptions, we establish simultaneous value and first-derivative approximation for separable targets with shared coordinate blocks and for general constrained holomorphic targets with anchored interaction blocks on compact polydiscs, and derive a mixed-derivative lower bound showing why coordinate-only representations cannot approximate joint log transforms arbitrarily well when stationary coordinates have nonzero cross-covariance. Combined with closed-form real-axis derivatives, sequential branch backpropagation, and dimension-dependent batch schedules, this design substantially reduces GPU memory demand and update time, allowing all tested benchmarks to run on a single GPU and scaling the 50-dimensional product-form benchmark to a single laptop GPU with stable accuracy across late-training checkpoints. Across degree-matched comparisons over the primary evaluation ranges, HCN attains lower reported mean relative errors than NN4MGF, including reducing error from $6.684\%$ to $1.422\%$ on the two-dimensional Harrison non-product benchmark, while an exploratory ten-dimensional correlated example illustrates the practical applicability of the interaction extension beyond two dimensions. Together with exact-transform and signed-error diagnostics, these findings support the joint design of holomorphic network architectures, memory-aware training schedules, and numerical Laplace inversion for BAR-based stationary-tail solvers.

\FloatBarrier
\bibliographystyle{plainnat}
\bibliography{references}

\clearpage
\appendix
\numberwithin{equation}{section}
\numberwithin{table}{section}
\numberwithin{figure}{section}
\renewcommand{\thelemma}{\Alph{section}.\arabic{lemma}}
\section*{Appendices}
\noindent The appendices provide mathematical details, constrained approximation proofs, and reference formulas (Appendix A), implementation settings (B), inversion diagnostics (C), supplementary results (D), and reproducibility information (E).
\section{Mathematical and Architectural Details}
\label{app:math}

\subsection{Stationary balance and boundary measures}
\label{app:bar-derivation}

For $h_\theta(z)=e^{-\theta^{\mathsf T}z}$, the interior generator and reflection derivatives are $\mathcal Ah_\theta(z)/h_\theta(z)=\tfrac12\theta^{\mathsf T}\Sigma\theta-\mu^{\mathsf T}\theta$ and $\nabla h_\theta(z)^{\mathsf T}R^{(k)}/h_\theta(z)=-\theta^{\mathsf T}R^{(k)}$. Under the integrability conditions in Section~\ref{sec:rbm}, stationarity and It\^o's formula imply
\begin{equation}
 0=\left(\tfrac12\theta^{\mathsf T}\Sigma\theta-\mu^{\mathsf T}\theta\right)\phi_0(\theta)-\sum_{k=1}^d\theta^{\mathsf T}R^{(k)}\phi_k(\theta),
 \label{eq:app-generator}
\end{equation}
giving \eqref{eq:bar} with the signs in \eqref{eq:gammas}. The same stationarity argument applied to the coordinates gives $Rm=-\mu$. Each boundary measure $\nu_k(A)=\mathbb E_\pi[\int_0^1\mathbf1_{\{Z(u)\in A\}}\,dY_k(u)]$ has total mass $m_k$ and support on $\{z:z_k=0\}$, so $\phi_k$ is a finite-measure Laplace transform with prescribed mass $m_k$ at the origin ($m_k=k$ in the product-form family; $m_k=1$ in the 2D Harrison and orthogonally reflected 10D models). Moreover, the boundary-isolation directions in \eqref{eq:isolated} satisfy $R^{\mathsf T}v_k=u_k/(R^{-\mathsf T})_{kk}$, confirming that all boundary coefficients except the $k$th vanish at $y_kv_k$ (simplifying to $y_ku_k$ when $R=I$).

\subsection{Proofs of the constrained approximation results}
\label{app:proofs}

We use the nested polydisc domains and first-derivative norm $\|\cdot\|_{1,K}$ defined in Section~\ref{sec:approx}. All derivatives in this subsection are complex derivatives, and the construction uses the exact rational Cauchy activation and scalar block parameterization in \eqref{eq:cauchy}--\eqref{eq:block}.

\begin{lemma}[Polynomial approximation by Cauchy blocks]\label{lem:poly-approx}
Let $P$ be a multivariate polynomial on $\CC^p$, $L\subset\CC^p$ a compact set, and $s_*>s_{\min}$. For every $\delta>0$, there exists a finite-width Cauchy block $q$ whose value and first-derivative errors relative to $P$ are strictly less than $\delta$ on $L$. Moreover, every hidden scale parameter can be set to $s_r=s_*$, with every raw denominator satisfying $|z^2+s_*^2|\ge \tfrac34 s_*^2$ on a neighborhood of $L$.
\end{lemma}

\begin{proof}
Set $a_r=-\mathrm{i}s_*$ and $b_r=s_*^2$ in the rational Cauchy activation \eqref{eq:cauchy}. For preactivation $z=\mathrm{i}s_* t$ with $|t|\le 1/2$, factoring the denominator gives
\begin{equation}\label{eq:app-pole-margin}
\sigma(\mathrm{i}s_* t)=\frac{1+t}{1-t^2}=\frac{1}{1-t},\qquad
|z^2+s_*^2|=s_*^2|1-t^2|\ge\frac34 s_*^2.
\end{equation}
Both raw poles $t=\pm 1$ lie strictly outside $|t|\le 1/2$, so the implemented fraction is smooth and bounded away from zero throughout the disk, and $s_*$ is realizable via $\rho_r=\log(\exp(s_*-s_{\min})-1)$ since $s_*>s_{\min}$.

For any integer degree $n\ge 1$, choose $N>n$, step $h>0$, and primitive root of unity $\omega=e^{2\pi\mathrm{i}/N}$. By the geometric series and roots-of-unity orthogonality filter,
\begin{equation}\label{eq:app-roots-filter}
\frac{1}{Nh^n}\sum_{\ell=0}^{N-1}\frac{\omega^{-\ell n}}{1-h\omega^\ell u}
=\sum_{q=0}^\infty h^{qN}u^{n+qN}
=\frac{u^n}{1-(hu)^N},\qquad |hu|<1.
\end{equation}
For $|u|\le A$ and $hA<1$, the uniform error relative to the monomial $u^n$ is bounded by $A^n(hA)^N/(1-(hA)^N)\to 0$ as $h\downarrow 0$. Each summand in \eqref{eq:app-roots-filter} is realized by a single Cauchy unit using preactivation $\mathrm{i}s_* h\omega^\ell u$ and output weight $\omega^{-\ell n}/(Nh^n)$. Any affine linear form $u(x)$ is realized by the complex input layer, and choosing $hA\le 1/2$ guarantees the uniform denominator margin \eqref{eq:app-pole-margin}.

Finally, powers of linear forms span all multivariate polynomials via the polarization identity:
\begin{equation}\label{eq:app-polarization}
x_1\cdots x_n=\frac{1}{n!}\sum_{S\subseteq\{1,\ldots,n\}}(-1)^{n-|S|}\biggl(\sum_{j\in S}x_j\biggr)^n,
\end{equation}
which remains valid when variables are repeated coordinates. Summing the finitely many resulting Cauchy units approximates $P$ uniformly on any compact polydisc, with constant terms absorbed into the output bias $c_2$. Enclosing $L$ in the interior of a slightly larger compact polydisc, uniform value convergence on the outer polydisc implies simultaneous first-derivative convergence on $L$ by Cauchy's integral estimates, while choosing $h$ sufficiently small so that $|h\omega^\ell u|\le 1/2$ on the outer polydisc preserves the denominator margin $\tfrac34 s_*^2$.
\end{proof}

\noindent\textbf{Interpolation of shared labels.} For any finite family of pairwise distinct boundary embeddings $b_k\in\CC^{p_b}$, choose a complex linear functional $\lambda$ such that the scalar projections $\beta_k=\lambda(b_k)$ are pairwise distinct (which exists since the finitely many excluded equalities define proper hyperplanes). The univariate Lagrange polynomials $D_k(b)=\prod_{\ell\ne k}(\lambda(b)-\beta_\ell)/(\beta_k-\beta_\ell)$ satisfy $D_k(b_\ell)=\mathbf1_{k=\ell}$. Consequently, for any family of polynomials $P_k(z)$, the single joint polynomial $Q(z,b)=\sum_k D_k(b)P_k(z)$ interpolates all slices simultaneously, and Lemma~\ref{lem:poly-approx} approximates $Q$ and its $z$-derivatives using a single shared Cauchy block on any compact envelope of the labelled inputs. An identical construction provides interpolating polynomials $E_j(e)$ for the distinct coordinate embeddings $e_j\in\CC^{p_e}$.

\begin{proof}[Proof of Theorem~\ref{thm:approx}(i)]
By Taylor's theorem and Cauchy's integral estimates on an intermediate polydisc inside $U$, univariate polynomials $P_j$ and $P_{kj}$ approximate each target component $h_j$ and $h_{kj}$, together with its first derivative, on the corresponding closed coordinate disc of $K$. Forming the joint interpolating polynomials $Q_{\rm int}(z,e)=\sum_{j=1}^d E_j(e)P_j(z)$ and $Q_{\rm bd}(z,e,b)=\sum_{k=1}^d\sum_{j\ne k}D_k(b)E_j(e)P_{kj}(z)$, we apply Lemma~\ref{lem:poly-approx} on a compact envelope containing all required labelled inputs and origin anchors (since $0$ lies in every coordinate disc of $K$) so that each slice of $q_{\rm int}$ and $q_{\rm bd}$ has value and first-derivative error below $\delta>0$. For any slice with target satisfying $h(0)=0$, writing $E(z)=q(z)-h(z)$ gives anchored error $E(z)-E(0)$, which is bounded by $2\delta$ in value and $\delta$ in derivative. Summing at most $d$ coordinate slices yields total value error below $2d\delta$ and coordinate derivative error below $\delta$; choosing $\delta<\eta/(2d)$ establishes \eqref{eq:approx-bound} while preserving \eqref{eq:identities} and the uniform denominator margin.
\end{proof}

\begin{proof}[Proof of Theorem~\ref{thm:approx}(ii)]
Let $B_0,B_1,\ldots,B_d$ denote the fixed coordinate-backbone functions (including $\log m_k$), which are holomorphic on $U$ and satisfy \eqref{eq:identities}. The residual targets $A_k=F_k-B_k$ are holomorphic on $U$ and vanish at the origin, with $A_k$ depending only on $\theta_{-k}\in K_{-k}$ for $k\ge 1$. Choose multivariate Taylor polynomials $P_0,P_1,\ldots,P_d$ approximating $A_0,A_1,\ldots,A_d$ to within $\delta$ in $\|\cdot\|_{1,K}$ on $K,K_{-1},\ldots,K_{-d}$. Applying Lemma~\ref{lem:poly-approx} approximates $P_0$ by $g_{\rm int}$ within $\delta$, and interpolating $P_1,\ldots,P_d$ via $Q(z,b)=\sum_{k=1}^d D_k(b)P_k(z)$ on an enclosing compact polydisc in $\CC^{d-1}\times\CC^{p_b}$ approximates all boundary residuals by a single shared block $g_{\rm bd}$ within $\delta$. On each slice, the unanchored error $E=g-A$ has value and first-derivative bounds below $2\delta$; since $A(0)=0$, origin anchoring yields error $E(\theta)-E(0)$ bounded by $4\delta$ in value and $2\delta$ in derivative. Setting $\delta<\eta/4$ proves \eqref{eq:approx-bound}.
\end{proof}

\noindent\textbf{Coordinate-only obstruction.} Applying the two-variable Cauchy integral formula to $F_0-f_0^{\rm coord}$ along the $i$th and $j$th coordinate circles of $K$ (holding other coordinates at $c$) and using $\partial_i\partial_j f_0^{\rm coord}\equiv 0$ for $i\ne j$ yields $|\partial_i\partial_j F_0(c)|\le (r_i r_j)^{-1}\sup_K|F_0-f_0^{\rm coord}|$, which rearranges to \eqref{eq:obstruction}. Since $\phi_0(0)=1$, $\partial_i\phi_0(0)=-\E_\pi[Z_i]$, and $\partial_i\partial_j\phi_0(0)=\E_\pi[Z_i Z_j]$, differentiating $F_0=\log\phi_0$ at $c=0$ gives $\partial_i\partial_j\log\phi_0(0)=\operatorname{Cov}_\pi(Z_i,Z_j)$.

\subsection{Consequences for transforms and the training constraints}
\label{app:loss-consequences}

Let $\phi_k=e^{F_k}$ and $\widehat\phi_k=e^{f_k}$, and suppose the log-transform approximation bound \eqref{eq:approx-bound} holds with tolerance $\eta>0$. Writing $E_k=f_k-F_k$ and using $|e^{E_k}-1|\le e^{|E_k|}-1$ on $K$ yields
\begin{equation}\label{eq:app-val-bound}
 |\widehat\phi_k-\phi_k|\le |\phi_k|(e^\eta-1),\qquad
 |\partial_j\widehat\phi_k-\partial_j\phi_k|\le |\phi_k|\bigl[(e^\eta-1)|\partial_j F_k|+e^\eta\eta\bigr].
\end{equation}
When the exact transforms satisfy the BAR identity \eqref{eq:bar} on $K$, subtracting \eqref{eq:bar} gives
\begin{equation}\label{eq:app-bar-bound}
 \Biggl|\gamma_0\widehat\phi_0-\sum_{k=1}^d\gamma_k\widehat\phi_k\Biggr|
 \le (e^\eta-1)\Biggl(|\gamma_0\phi_0|+\sum_{k=1}^d|\gamma_k\phi_k|\Biggr),
\end{equation}
whose right-hand multiplier is uniformly bounded on the compact set $K$. Under the common stabilizing shift $\nu_f(\theta)=\max_{0\le k\le d}\{\log(|\gamma_k(\theta)|+\varepsilon_\gamma)+\Re f_k(\theta)\}$, uniform convergence $f_k\to F_k$ on $K$ bounds $e^{-\nu_f}$ uniformly above, while the denominator of $\mathcal E$ in \eqref{eq:barloss} is bounded below by $\varepsilon^2$ and the weights $w_n\in[0.5,2]$ are bounded. Consequently, \eqref{eq:app-bar-bound} implies $\mathcal L_{\rm BAR}\to 0$ and $\mathcal L_{\rm bd}\to 0$ as $\eta\downarrow 0$. Likewise, on any real sample support in $K$ where $F_k$ is real and coordinatewise nonincreasing, $|\Im f_k|\le\eta$ and $\max\{0,\Re\partial_j f_k\}\le\eta$, so $\mathcal L_{\rm mono}\le\eta^2$ and $\mathcal L_{\rm imag}\le\eta^2$. Thus all four loss components in \eqref{eq:full-loss} tend to zero simultaneously as $\eta\downarrow 0$.

\subsection{Parameter counts and interaction allocations}
\label{app:parameter-count}

A scalar Cauchy block with $p$ complex inputs and $H$ hidden units contains $pH$ complex input weights, $H$ complex hidden biases, $H$ complex output weights, one complex output bias, two complex activation coefficients per unit, and one real scale parameter per unit. Counting each complex scalar as two real scalars gives $P_{\rm Cauchy}(p,H)=2(pH+H+H+1+2H)+H=H(2p+9)+2$. With $p_e=p_b=16$, the coordinate and boundary embeddings contribute $64d$ real parameters, and the coordinate interior and boundary input dimensions are $17$ and $33$, respectively. Setting $H=128$ yields $15{,}108+64d$ real parameters. A complex-tanh block has $H(2p+4)+2$ real parameters since its activation contains no trainable coefficients.

\begin{table}[!htbp]
\centering\small
\caption{Hidden widths of the scalar subnetworks and real scalar parameter counts.}
\label{tab:app-parameters}
\begin{tabular}{lrrrrr}
\toprule
Model & $q_{\rm int}$ & $q_{\rm bd}$ & $g_{\rm int}$ & $g_{\rm bd}$ & Parameters\\
\midrule
2D product & 128 & 128 & --- & --- & 15,236\\
2D Harrison & 95 & 128 & 109 & --- & 15,236\\
10D non-product & 95 & 128 & 109 & 109 & 23,925\\
20D product, Cauchy & 128 & 128 & --- & --- & 16,388\\
20D product, tanh & 128 & 128 & --- & --- & 15,108\\
30D product & 128 & 128 & --- & --- & 17,028\\
50D product & 128 & 128 & --- & --- & 18,308\\
\bottomrule
\end{tabular}
\par\vspace{3pt}\begin{minipage}{\linewidth}\footnotesize The four middle columns list hidden-unit counts for the coordinate interior, coordinate boundary, interior interaction, and boundary interaction blocks. Parameter totals include embeddings and count each complex parameter as two real scalars.\end{minipage}
\end{table}

In the 2D Harrison model, the two embedding arrays contribute 128 real parameters, while $q_{\rm int}$, $q_{\rm bd}$, and $g_{\rm int}$ contribute $4{,}087$, $9{,}602$, and $1{,}419$ parameters, matching the $15{,}236$-parameter budget of the 2D product-form model. In the 10D correlated model, the corresponding contributions are $640+4{,}087+9{,}602+3{,}163+6{,}433=23{,}925$, where the final term is the shared boundary interaction block $g_{\rm bd}$ with input dimension $9+16=25$.

\subsection{Analytical and density-based references}
\label{app:reference-formulas}

For the product-form family, $R^{-1}\mathbf1=(1,\ldots,d)^{\mathsf T}$ and $\Sigma=R+R^{\mathsf T}$, giving independent exponential stationary coordinates with rates $1,\ldots,d$ (Harrison and Williams, 1987). Letting $F_d(t)=(1-e^{-t})^d$ for $t\ge0$ and substituting $u=e^{-t}$ yields $\int_0^\infty e^{-st}\,dF_d(t)=d\int_0^1 u^s(1-u)^{d-1}\,du=\Gamma(d+1)\Gamma(s+1)/\Gamma(s+d+1)=\prod_{i=1}^d i/(i+s)$ for $\operatorname{Re}s>0$, confirming \eqref{eq:product-true}. In floating-point evaluation, computing $p_d(t)=-\operatorname{expm1}\{d\operatorname{log1p}(-e^{-t})\}$ avoids catastrophic cancellation at large thresholds.

For the 2D Harrison benchmark \eqref{eq:harrison}, the exact stationary density (Dai and Harrison, 1992; Dai and Zhang, 2026) in polar coordinates $(z_1,z_2)=(r\cos\alpha,r\sin\alpha)$ is
\begin{equation}
 \varrho_{\rm H}(r\cos\alpha,r\sin\alpha)
 =\frac{2^{3/2}}{\sqrt\pi}\,r^{-1/2}
   e^{-(1+\cos\alpha)r}\cos(\alpha/2),
 \quad r>0,\quad 0\le\alpha\le\pi/2.
 \label{eq:app-harrison-density}
\end{equation}
With area element $dz_1\,dz_2=r\,dr\,d\alpha$, the tail event $z_1+z_2>t$ imposes the radial lower limit $r_0(\alpha)=t/(\cos\alpha+\sin\alpha)$. Integrating out $r$ analytically via the substitution $u=(1+\cos\alpha)r$ yields the one-dimensional integral representation
\begin{equation}
 p_{\rm H}(t)=\frac{2^{3/2}}{\sqrt\pi}
 \int_0^{\pi/2}
 \frac{\cos(\alpha/2)}{(1+\cos\alpha)^{3/2}}
 \Gamma\!\left(\frac32,
       \frac{(1+\cos\alpha)t}{\cos\alpha+\sin\alpha}\right)d\alpha,
 \label{eq:app-harrison-reference}
\end{equation}
where $\Gamma(a,x)=\int_x^\infty u^{a-1}e^{-u}\,du$ is the upper incomplete gamma function. High-precision quadrature of \eqref{eq:app-harrison-reference} supplies the exact reference probabilities at $t=1,\ldots,5$, agreeing with the rounded values in Dai and Zhang (2026) to within $6\times10^{-8}$.

\section{Training and Implementation Details}
\label{app:implementation}

\subsection{Stable residuals}
\label{app:stable-loss}
For $\gamma_{nk}=\gamma_k(\theta^{(n)})$ and $f_{nk}=f_k(\theta^{(n)})$, the common real shift in \eqref{eq:kappas} is $\nu_n=\max_{0\le k\le d}\{\log(|\gamma_{nk}|+\varepsilon_\gamma)+\operatorname{Re}f_{nk}\}$ with $\varepsilon_\gamma=10^{-12}$. This normalization bounds each scaled term $|\kappa_{nk}|$ by one in exact arithmetic while preserving the zero-residual locus of the BAR equation. The relative residual in \eqref{eq:barloss} uses $\varepsilon=10^{-8}$. For $\ell_n=\mathcal E(A_n,B_n)$, stop-gradient operator $\operatorname{sg}$, and $u_n=(\operatorname{sg}(\ell_n)+\varepsilon_w)^{0.3}$ with $\varepsilon_w=10^{-12}$, the adaptive sample weights are
\begin{equation}
 w_n=\operatorname{sg}\!\left[
 \min\left\{2,\max\left\{0.5,
 \frac{u_n}{N_{\rm BAR}^{-1}\sum_m u_m+\varepsilon_w}\right\}\right\}\right].
 \label{eq:app-adaptive-weights}
\end{equation}
At isolated boundary arguments \eqref{eq:isolated}, the stabilizing shift takes the maximum over the two active terms $\log(|\gamma_j|+\varepsilon_\gamma)+\operatorname{Re}f_j$ for $j\in\{0,k\}$, and the pairwise boundary loss \eqref{eq:bdloss} takes the unweighted mean across samples and boundary faces.

\subsection{Sampling and numerical settings}
Following \citet{DaiZhang2026}, the general complex sampler draws points from random nested boxes on a bounded domain $\Theta=[\underline{\theta}^{\rm Re},\overline{\theta}^{\rm Re}]^d\times[-\overline{\theta}^{\rm Im},\overline{\theta}^{\rm Im}]^d$. For each base input, we independently draw $U,V\sim\operatorname{Unif}(0,1)$, set the box thresholds $A=\underline{\theta}^{\rm Re}+(\overline{\theta}^{\rm Re}-\underline{\theta}^{\rm Re})U$ and $C=\overline{\theta}^{\rm Im}V$, and draw coordinates conditionally independently as $\operatorname{Re}\theta_j\mid A\sim\operatorname{Unif}(\underline{\theta}^{\rm Re},A)$ and $\operatorname{Im}\theta_j\mid C\sim\operatorname{Unif}(-C,C)$ for $j=1,\ldots,d$. Sharing $(A,C)$ across coordinates concentrates a higher fraction of samples in corner regimes where all coordinates simultaneously take small or large values. Complex-diagonal inputs repeat a single univariate complex draw across all $d$ coordinates; near-diagonal inputs add independent real and imaginary perturbations from $[-0.05,0.05]$ and project back into $\Theta$; and nonnegative real-diagonal inputs take $\overline{\theta}^{\rm Re}U^2\mathbf1$. Each interior batch $B_{\rm BAR}$ allocates integer parts of 50\%, 20\%, and 20\% to the first three components and assigns the remainder to the real diagonal. Boundary base vectors use the general complex sampler before applying \eqref{eq:isolated}, and real-axis samples use the real-part projection of the nested-box sampler.

\begin{table}[!htbp]
\centering\small
\caption{Implementation hyperparameters and numerical settings.}
\label{tab:app-implementation}
\begin{tabularx}{\linewidth}{lX}
\toprule
Setting & Value\\
\midrule
Arithmetic & \texttt{complex128} network values and weights; \texttt{float64} real quantities.\\
Embeddings & $p_e=p_b=16$ complex components per coordinate and boundary face, initialized $\mathcal N(0,0.01^2)$.\\
Cauchy initialization & $a_r=b_r=0.1$, $s_r=0.5$, $s_{\min}=10^{-3}$; $\rho_r=\log(\exp(0.5-s_{\min})-1)$.\\
Optimizer & Adam, learning rate $5\times10^{-5}$ for the first 80\% of updates, $10^{-5}$ thereafter.\\
Loss and gradients & Unit weights $(1,1,1,1)$ in \eqref{eq:full-loss}; global accumulated gradient norm clipped at $1.0$.\\
Pair chunk cap & 32,768 scalar inputs in 2D and product-form models; 4,096 in 10D.\\
Backward schedule & Combined pass in 2D and 20D; sequential three-branch pass in 30D and 50D; facewise boundary accumulation in 10D.\\
\bottomrule
\end{tabularx}
\end{table}

Complex affine layers use \texttt{torch.nn.Linear} with \texttt{dtype=complex128} and standard Kaiming fan-in initialization. During sequential branch accumulation (Algorithm~1), all loss branches evaluate the exact same parameter vector $\omega$ prior to global gradient clipping and the single Adam update.

\subsection{Preflight verification checks}
\label{app:pole-checks}
Before training commences, automated preflight diagnostics verify numerical well-posedness and closed-form Jacobian accuracy at initialization. We monitor the minimum denominator modulus $\Delta_{\rm den}=\min|h_r(\theta)^2+s_r^2|$ across all hidden preactivations $h_r(\theta)$ and verify finiteness of all forward outputs, Jacobians, losses, and parameter gradients. In addition, the 2D and 10D routines verify the closed-form complex input derivatives \eqref{eq:cauchyprime}--\eqref{eq:fprime} against centered finite differences $D_j^{(h)}F(x)=(F(x+hu_j)-F(x-hu_j))/(2h)$ with step $h=10^{-6}$. As shown in Table~\ref{tab:app-preflight}, the maximum discrepancy between our explicit analytical Jacobians and centered finite differences is on the order of $10^{-10}$ to $10^{-9}$, well within double-precision finite-difference truncation tolerances. Origin-anchor identities $\widehat\phi_0(0)=1$, $\widehat\phi_k(0)=m_k$, and normal-coordinate independence $\partial_{\theta_k}f_k=0$ hold to machine precision ($10^{-12}$ to $10^{-10}$).

\begin{table}[!htbp]
\centering\small
\caption{Finite-difference Jacobian verification checks at initialization.}
\label{tab:app-preflight}
\setlength{\tabcolsep}{5pt}
\begin{tabular}{llrr}
\toprule
Benchmark & Checked object & Tolerance & Maximum error \\
\midrule
2D product & Coordinate block & $2\times10^{-7}$ & $1.131\times10^{-9}$ \\
2D Harrison & Coordinate block & $2\times10^{-7}$ & $2.99\times10^{-9}$ \\
2D Harrison & Interaction block & $2\times10^{-7}$ & $3.577\times10^{-10}$ \\
10D non-product & Full log-transform vector & $1\times10^{-5}$ & $1.818\times10^{-10}$ \\
\bottomrule
\end{tabular}
\par\vspace{3pt}\begin{minipage}{\linewidth}\footnotesize All checks use centered finite differences with step $h=10^{-6}$ and report the maximum absolute discrepancy across all tested complex probes and derivative entries.\end{minipage}
\end{table}

\section{Evaluation and Inversion Diagnostics}
\label{app:inversion}

\subsection{Evaluation implementation}
\label{app:inversion-implementation}
At the frozen training endpoints, $\widehat G_S(s)$ in \eqref{eq:learned-lt} is numerically inverted using \texttt{mpmath}'s \texttt{invertlaplace} routine (\texttt{method="talbot"}) at degree $M\in\{5,8\}$. Although the caller sets \texttt{mpmath.\allowbreak mp.dps} to 50 decimal digits in 2D and 10D and 100 digits in 20D--50D, passing an explicit \texttt{degree} $M$ causes \texttt{mpmath}\textquotesingle s \texttt{FixedTalbot} routine to set its internal quadrature precision from $M$ automatically, while all neural network queries run in \texttt{complex128}.

At zero, exact origin normalization $\widehat\phi_0(0)=1$ gives the removable singularity limit
\begin{equation}
 \widehat G_S(0)=-\sum_{j=1}^d\partial_j f_0(0).
 \label{eq:app-zero-limit}
\end{equation}
The exact product-form limit is $H_d=\sum_{j=1}^d 1/j$.

\subsection{Signed-error decomposition}
\label{app:signed-errors}
Let $p=p_{\rm ref}(t)>0$ be the reference tail probability, $p_M^\star=\min\{1,\max\{0,T_M[G_S](t)\}\}$ the clipped output obtained by inverting the exact transform under the same $[0,1]$ projection rule, and $\widehat p_M=\min\{1,\max\{0,T_M[\widehat G_S](t)\}\}$ the corresponding clipped learned-transform tail prediction. We define the signed relative error components in percent:
\begin{equation}
 r_{\rm inv}=100\frac{p_M^\star-p}{p},\qquad
 r_{\rm net}=100\frac{\widehat p_M-p_M^\star}{p},\qquad
 r_{\rm total}=r_{\rm inv}+r_{\rm net}=100\frac{\widehat p_M-p}{p}.
 \label{eq:app-signed-error}
\end{equation}
Pointwise relative error is $|r_{\rm total}|$, and $\MRE_{1:K}$ averages this quantity over $t\in\{1,\ldots,K\}$. Taking absolute values before summing the components would discard sign cancellation, where $r_{\rm net}$ measures the net effect of substituting the learned transform into the inversion pipeline.

\begin{table}[!htbp]
\centering\small
\caption{Signed relative-error decomposition at 20D ($t=8$, percent).}
\label{tab:app-signed-errors}
\setlength{\tabcolsep}{5pt}
\begin{tabular}{lrrrrr}
\toprule
Model & $M$ & $t$ & $r_{\rm inv}$ & $r_{\rm net}$ & $r_{\rm total}$ \\
\midrule
NN4MGF & 5 & 8 & -2.516 & +18.374 & +15.858 \\
HCN (32345) & 5 & 8 & -2.516 & -16.412 & -18.928 \\
NN4MGF & 8 & 8 & -0.447 & -28.582 & -29.028 \\
HCN (32345) & 8 & 8 & -0.447 & -8.037 & -8.484 \\
\bottomrule
\end{tabular}
\par\vspace{3pt}\begin{minipage}{\linewidth}\footnotesize Components are computed at full precision and rounded only for display, satisfying $r_{\rm total}=r_{\rm inv}+r_{\rm net}$.\end{minipage}
\end{table}

Table~\ref{tab:app-signed-errors} compares NN4MGF and HCN at the same tail threshold $t=8$ in 20 dimensions under Degrees~5 and~8. Under Degree~5 ($M=5$), the intrinsic quadrature error $r_{\rm inv}=-2.516\%$ has the opposite sign to NN4MGF's network error ($r_{\rm net}=+18.374\%$), partially canceling to yield a smaller total error ($+15.858\%$). By contrast, HCN (seed 32345) has a smaller network error magnitude ($|r_{\rm net}|=16.412\%$), but because its negative sign matches $r_{\rm inv}=-2.516\%$, the two components reinforce to give $-18.928\%$ (and similarly $-23.220\%$ for seed 22345). Switching to Degree~8 ($M=8$) at the exact same threshold $t=8$ reduces the intrinsic inversion error magnitude to $r_{\rm inv}=-0.447\%$, removing the artificial Degree-5 cancellation and revealing HCN's substantially lower total error ($-8.484\%$ for seed 32345 and $-14.963\%$ for seed 22345, compared with $-29.028\%$ for NN4MGF). For a linear quadrature rule $T_M[G]=\operatorname{Re}\sum_j w_j G(s_j)$, exact arithmetic gives $|T_M[\widehat G]-T_M[G]|\le\sum_j|w_j|\,|\widehat G(s_j)-G(s_j)|$, showing how quadrature weights $|w_j|$ and node locations $s_j$ govern the propagation of transform approximation errors across degrees.

\section{Supplementary Experimental Results}
\label{app:additional-results}

\subsection{Degree sensitivity at fixed endpoints}
\label{app:degree-sensitivity}
Table~\ref{tab:app-degree-sensitivity} pairs Degrees~5 and~8 at the fixed training endpoints across the 20D, 30D, and 50D product-form benchmarks. Within each pair, only the inversion degree changes; the trained network weights are identical. Moving from Degree~5 to Degree~8---which removes the large high-dimensional quadrature bias documented in Table~\ref{tab:exact-inversion}---reduces MRE across all three evaluation intervals ($1{:}5$, $1{:}8$, and $1{:}10$) for every HCN benchmark: the 20D HCN three-seed mean decreases from $2.983\%$ to $1.441\%$ on $1{:}5$, from $7.730\%$ to $7.632\%$ on $1{:}8$, and from $19.347\%$ to $12.264\%$ on $1{:}10$; the 30D HCN errors decrease from $3.429\%$, $10.107\%$, and $19.084\%$ to $1.683\%$, $6.546\%$, and $8.872\%$; and the 50D HCN errors decrease from $4.605\%$, $7.196\%$, and $13.182\%$ to $1.495\%$, $6.367\%$, and $8.011\%$. By contrast, for the archived 20D NN4MGF checkpoint, removing the Degree-5 sign cancellation increases $\MRE_{1:8}$ from $8.006\%$ to $8.148\%$ and increases $\MRE_{1:10}$ from $16.801\%$ to $18.908\%$.

\begin{table}[!htbp]
\centering\small
\caption{Fixed-endpoint MREs under Degrees 5 and 8 (percent).}
\label{tab:app-degree-sensitivity}
\setlength{\tabcolsep}{5pt}
\begin{tabular}{lrrrr}
\toprule
Model & $M$ & $1{:}5$ & $1{:}8$ & $1{:}10$ \\
\midrule
20D HCN (3-seed mean) & 5 & 2.983 & 7.730 & 19.347 \\
20D HCN (3-seed mean) & 8 & 1.441 & 7.632 & 12.264 \\
20D NN4MGF & 5 & 4.013 & 8.006 & 16.801 \\
20D NN4MGF & 8 & 2.802 & 8.148 & 18.908 \\
30D HCN & 5 & 3.429 & 10.107 & 19.084 \\
30D HCN & 8 & 1.683 & 6.546 & 8.872 \\
50D HCN & 5 & 4.605 & 7.196 & 13.182 \\
50D HCN & 8 & 1.495 & 6.367 & 8.011 \\
\bottomrule
\end{tabular}
\par\vspace{3pt}\begin{minipage}{\linewidth}\footnotesize Each degree pair evaluates the exact same trained network weights without retraining. The 20D HCN entries report the mean across the three trained seeds.\end{minipage}
\end{table}

\subsection{Late-training diagnostics in 50 dimensions}
\label{app:late-training}
Table~\ref{tab:app-late-training} supplements Section~\ref{sec:late-training} by reporting $\MRE_{1:5}$ and $\MRE_{1:8}$ across the five late-training checkpoints ($92{,}000$ to $100{,}000$ updates) of the 50D trajectory at Degree~8. Across these five checkpoints, $\MRE_{1:5}$ remains between $1.318\%$ and $1.611\%$, and $\MRE_{1:8}$ remains between $5.299\%$ and $7.155\%$ ($6.367\%$ at the fixed $100{,}000$-update endpoint).

\begin{table}[!htbp]
\centering\small
\caption{Late-training checkpoints along the 50D trajectory at Degree 8 (MRE, percent).}
\label{tab:app-late-training}
\setlength{\tabcolsep}{5pt}
\begin{tabular}{rrrl}
\toprule
Update & $1{:}5$ & $1{:}8$ & Role \\
\midrule
92,000 & 1.524 & 6.650 & Diagnostic checkpoint \\
94,000 & 1.611 & 5.299 & Diagnostic checkpoint \\
96,000 & 1.318 & 6.845 & Diagnostic checkpoint \\
98,000 & 1.473 & 7.155 & Diagnostic checkpoint \\
100,000 & 1.495 & 6.367 & Reported endpoint \\
\bottomrule
\end{tabular}
\par\vspace{3pt}\begin{minipage}{\linewidth}\footnotesize All five checkpoints from update 92,000 to 100,000 are evaluated under Degree 8. The main accuracy comparison uses the fixed 100,000-update endpoint.\end{minipage}
\end{table}

\subsection{The exploratory 10D non-product example}
\label{app:10d-reference}
In the 10-dimensional orthogonally reflected RBM \eqref{eq:10d-model}, orthogonal reflection ($R=I_{10}$) preserves the one-dimensional stationary marginal exponential laws $\mathbb P_\pi(Z_i>x)=e^{-ix}$ ($\mathbb E_\pi[Z_i]=1/i$ and $\phi_0(su_i)=i/(i+s)$) despite correlated Brownian input, giving exact mean total workload $\mathbb E_\pi[S]=H_{10}=2.9289682540\ldots$. Evaluating the removable-limit derivative formula \eqref{eq:app-zero-limit} for the frozen HCN network gives a learned mean total workload of $2.9242988034$, a $0.159\%$ discrepancy relative to the exact analytical mean $H_{10}$.

To generate a reproducible steady-state sum-tail reference for \eqref{eq:10d-model}, we use the time-reversal identity for orthogonally reflected Brownian motion ($R=I_{10}$): the stationary vector $Z$ has the joint law of the coordinatewise all-time running maxima $M_i(\infty)=\sup_{s\ge 0}X_i(s)$, where $X_i(s)=B_i(s)-s$ and $\operatorname{Cov}(B(s),B(u))=\min(s,u)\Sigma$. Truncating to horizon $T=64$ incurs a union-bound tail error of at most $q_T=\sum_{i=1}^{10}\mathbb P(\sup_{s\ge T}X_i(s)>0)=\sum_{i=1}^{10}\operatorname{erfc}(\sqrt{iT/4})\approx 1.542\times 10^{-8}$. On $[0,T]$, we simulate $1{,}000{,}000$ independent correlated paths (seed \texttt{2026102804}) from a dyadic skeleton and recursively refine correlated Brownian bridges on subintervals until no threshold $t\in\{1,\ldots,8\}$ lies between the pathwise lower sum $L=\sum_{i=1}^{10}l_i$ of observed endpoints and the upper sum $U=\sum_{i=1}^{10}U_i$ of conditional Brownian-bridge upper envelopes (chosen so that the union-bound probability of exceeding $U$ across all coordinates and subintervals is at most $\delta=10^{-7}$ per path), or until $U-L\le 0.005$. Across paths, exact Clopper--Pearson binomial limits applied to the lower counts $\mathbf{1}\{L>t\}$ and upper counts $\mathbf{1}\{U>t\}$ (inflated by $\delta+q_T$) yield the pointwise 95\% confidence intervals in Table~\ref{tab:sec5-nonproduct}, while the reference point estimate is the midpoint of the lower and upper empirical tail probabilities, $\tfrac12\bigl(\widehat{\mathbb P}(L>t)+\widehat{\mathbb P}(U>t)\bigr)$ (\path{scripts/reference_10d/}). An independent calibration at $\rho=0$ over $50{,}000$ paths verifies that all eight analytical sum-tail probabilities $p_{\rm ind}(t)$ fall within their simultaneous 95\% intervals, and an earlier pilot simulation estimate used in preliminary drafts falls inside all eight pointwise and simultaneous 95\% intervals (maximum relative difference $0.652\%$, corresponding to $3.684\%$ versus $3.829\%$ HCN MRE).

\begin{table}[!htbp]
\centering\small
\caption{Ten-dimensional correlated non-product RBM at Degree 8.}
\label{tab:sec5-nonproduct}
\setlength{\tabcolsep}{3.5pt}
\begin{tabular}{rrcrrrr}
\toprule
$t$ & Sim.\ ref. & 95\% CI & HCN & Independence & HCN RE (\%) & Indep.\ RE (\%) \\
\midrule
1 & 0.9726420 & [0.972021, 0.973257] & 1.000000 & 0.989814 & 2.81 & 1.77 \\
2 & 0.7226200 & [0.720892, 0.724346] & 0.722273 & 0.766398 & 0.05 & 6.06 \\
3 & 0.4011860 & [0.399547, 0.402826] & 0.398938 & 0.399920 & 0.56 & 0.32 \\
4 & 0.1911600 & [0.189993, 0.192330] & 0.187112 & 0.168775 & 2.12 & 11.71 \\
5 & 0.0841285 & [0.083409, 0.084851] & 0.086949 & 0.065373 & 3.35 & 22.29 \\
6 & 0.0355610 & [0.035135, 0.035990] & 0.039650 & 0.024513 & 11.50 & 31.07 \\
7 & 0.0147035 & [0.014443, 0.014967] & 0.015140 & 0.009081 & 2.97 & 38.24 \\
8 & 0.0059435 & [0.005781, 0.006109] & 0.005512 & 0.003350 & 7.27 & 43.64 \\
\midrule Mean ($1{:}8$) & --- & --- & --- & --- & 3.83 & 19.39 \\
\bottomrule
\end{tabular}
\par\vspace{3pt}\begin{minipage}{\linewidth}\footnotesize Simulation references and pointwise 95\% Clopper--Pearson confidence intervals (incorporating analytical finite-horizon and Brownian-bridge bounds) are computed from $1{,}000{,}000$ independent correlated paths on $T=64$. HCN and independent-marginal comparator MREs are 3.829\% and 19.386\%, respectively (conservative simultaneous 95\% MRE intervals $[2.539\%, 5.142\%]$ and $[18.504\%, 20.267\%]$). First-moment cross-validation against $\mathbb{E}_\pi[S]=H_{10}$ agrees to within 0.159\%.\end{minipage}
\end{table}

\section{Reproducibility and Computational Accounting}
\label{app:reproducibility}

\subsection{Experimental records and provenance}
\label{app:provenance}
Archived evaluation records in \path{paper/data/}, \path{paper/audit/}, and \path{results/} store the complete pointwise predictions, hardware profiles, configurations, and SHA-256 checksums for every benchmark, seed, checkpoint, and Talbot degree reported in the paper. Baseline NN4MGF comparisons use the pointwise 2D and 30D tail predictions released in the official repository of Dai and Zhang (2026) (\url{https://github.com/zhangz73/NN4MGF}) and full-precision evaluations of the released 20D NN4MGF checkpoint. In the 30D product-form experiment, training proceeds for $50{,}000$ updates with learning-rate annealing at update $40{,}000$ under the audited 64M base-input budget.

\subsection{Software environment}
\label{app:software}
Training and evaluation runs on the workstation/laptop platform use Python 3.10.20, PyTorch 2.5.1+cu121, CUDA 12.1, and \texttt{mpmath} 1.3.0 on an NVIDIA GeForce RTX 4080 Laptop GPU, while the controlled 30-dimensional hardware resource profiles in Section~\ref{sec:performance} are executed on a common NVIDIA H100 NVL GPU. The repository archives the reference environment files (\texttt{environment.yml} and \texttt{requirements.txt}), preflight check logs (\texttt{paper/provenance/}), and per-run configuration records under \texttt{results/} and \texttt{benchmarks/}.

\subsection{Resource and sampling-budget accounting}
\label{app:resource-accounting}
In the 30D H100 profile (Table~\ref{tab:profile}), from 200 post-warmup update times $\tau_j$, the mean update time is $\overline\tau=200^{-1}\sum_{j=1}^{200}\tau_j$ ($0.165113$~s for HCN and $1.601046$~s for NN4MGF, totaling $33.023$~s and $320.209$~s), the projected full-schedule time is $\widehat T=U\overline\tau$, and peak GPU memory ($\text{GiB}=2^{30}$ bytes) is recorded from CUDA statistics reset after warmup. Under \eqref{eq:budget}, HCN processes $1{,}024+128+128=1{,}280$ base inputs per update over $U=50{,}000$ updates ($N_{\rm HCN}=64{,}000{,}000$), whereas the released NN4MGF code samples two batches of $16{,}384$ inputs over $U=300{,}000$ updates ($N_{\rm NN4MGF}=9{,}830{,}400{,}000$); HCN's 64M budget is $0.651\%$ of this code-path count ($1.302\%$ of the single-batch count $4.9152\text{B}$ in Dai and Zhang, 2026).

\subsection{Code availability and reproduction scripts}
\label{app:regeneration}
Source code, trained checkpoints, and reproduction scripts for HCN are publicly available at \url{https://github.com/KehanWang126/holomorphic-cauchy-networks-rbm}. From the repository root, the entry-point scripts in Table~\ref{tab:app-regeneration} verify all SHA-256 data checksums and regenerate the experimental numerical summaries, result tables, and figures in the main text and appendices directly from the archived evaluation records.
\begin{table}[!htbp]
\centering\footnotesize
\caption{Repository entry points for verifying records and rebuilding manuscript tables and figures.}
\label{tab:app-regeneration}
\begin{tabularx}{\linewidth}{l>{\raggedright\arraybackslash}X}
\toprule
Task & Repository entry point\\
\midrule
Verify SHA-256 hashes and audit summaries & \path{paper/audit/audit_results.py}, \path{scripts/build_paper_summaries.py}\\
Rebuild main and appendix result tables and figures & \path{paper/build_results.py}, \path{paper/scripts/build_appendix_tables.py}\\
Validate or rerun 10D simulation reference & \path{scripts/reference_10d/validate_records.py}, \path{scripts/reference_10d/compare_reference.py}, \path{scripts/reference_10d/run_reference.py}\\
\bottomrule
\end{tabularx}
\end{table}

\FloatBarrier
\end{document}